%% file: main_paper.tex
\documentclass[11pt,letterpaper]{article} %
\usepackage{graphicx} %

\title{
Of Dice and Games: \\ A Theory of Generalized Boosting
}

\date{ }
\usepackage{amsmath}
\usepackage{amsfonts, amssymb}
\usepackage{mathtools}
\usepackage{amsthm} %
\usepackage{latexsym}
\usepackage{relsize}
\usepackage{color}
\usepackage{nicefrac}
\usepackage{natbib}
\usepackage{comment}
\usepackage{booktabs}
\usepackage[linesnumbered,ruled,vlined,algo2e,noend]{algorithm2e}
\usepackage{titlesec}
\usepackage{footnotebackref}

\renewcommand{\mathbf}{\boldsymbol}

\usepackage[margin=1in]{geometry} %

\input{header}

\author{
    Marco Bressan$^{1}$, 
    Nataly Brukhim$^{2}$, 
    Nicolò Cesa-Bianchi$^{1,3}$, \\
    Emmanuel Esposito$^{1}$, 
    Yishay Mansour$^{4,5}$, 
    Shay Moran$^{6,5}$, 
    Maximilian Thiessen$^{7}$
    \\[1em]  %
    \small $^1$Universita degli Studi di Milano, Italy \qquad
    \small $^2$Institute for Advanced Study, USA \\ 
    \small $^3$Politecnico di Milano, Italy \qquad
    \small $^4$Tel Aviv University, Israel \qquad
    \small $^5$Google Research \\ 
    \small $^6$Technion, Israel \qquad 
    \small $^7$TU Wien, Austria
}

\begin{document}

\maketitle

\begin{abstract}
Cost-sensitive loss functions are crucial in many real-world prediction problems, where different types of errors are penalized differently; for example, in medical diagnosis, a false negative prediction can lead to worse consequences than a false positive prediction. However, traditional PAC learning theory has mostly focused on the symmetric 0-1 loss, leaving cost-sensitive losses largely unaddressed.
In this work, we extend the celebrated theory of boosting to incorporate both cost-sensitive and multi-objective losses.
Cost-sensitive losses assign costs to the entries of a confusion matrix, and are used to control the sum of prediction errors accounting for the cost of each error type. Multi-objective losses, on the other hand, simultaneously track multiple cost-sensitive losses,
and are useful when the goal is to satisfy several criteria at once (e.g., minimizing false positives while keeping false negatives below a critical threshold).

We develop a comprehensive theory of cost-sensitive and multi-objective boosting, providing a taxonomy of weak learning guarantees that distinguishes which guarantees are trivial (i.e., can always be achieved), which ones are boostable (i.e., imply strong learning), and which ones are intermediate, implying non-trivial yet not arbitrarily accurate learning. For binary classification, we establish a dichotomy: a weak learning guarantee is either trivial or boostable. In the multiclass setting, we describe a more intricate landscape of intermediate weak learning guarantees. Our characterization relies on a geometric interpretation of boosting, revealing a surprising equivalence between cost-sensitive and multi-objective losses.

\smallskip\noindent\textbf{Keywords:} Boosting, Minimax theorem, cost-sensitive learning, multi-objective learning, Blackwell's approachability.
\end{abstract}

\thispagestyle{empty}
\newpage 
\setcounter{page}{1}

\input{1-intro}

\input{2-results}

\newpage
\input{3-binary}
\input{4-duality}

\input{5-multi}

\clearpage
\phantomsection
\bibliographystyle{plainnat}
\bibliography{bib}

\appendix
\input{app-cost-sensitive}

\end{document}

%% file: header.tex
\usepackage{amsmath,amsfonts, amssymb}
\usepackage{mathtools}
\usepackage{amsthm} %
\usepackage{latexsym}
\usepackage{relsize}
\usepackage{tikz}
\usetikzlibrary{calc}

\usepackage{graphicx}
\usepackage{subcaption}

\usepackage{tabularx}

\usepackage{multirow,array}
\usepackage{floatrow}

\usepackage[font=small,labelfont=bf,tableposition=top]{caption}

\DeclareCaptionLabelFormat{andtable}{#1~#2  \&  \tablename~\thetable}

\usepackage{subfig}

\renewcommand{\mathbf}{\boldsymbol}

\usepackage{tikz}
\usetikzlibrary{positioning}
\usepackage{makecell}
\usetikzlibrary{fit}
\usetikzlibrary{shapes.geometric}

\usepackage{natbib}
\usepackage[nottoc]{tocbibind}

\usepackage{hyperref}       %
\usepackage{url}            %
\usepackage{algorithm}
\usepackage{algorithmic}

\usepackage{braket}

\usepackage{booktabs}       %
\usepackage{amsfonts}       %
\usepackage[shortlabels]{enumitem}
\usepackage[bbgreekl]{mathbbol}
\usepackage{xcolor}
\hypersetup{
    colorlinks,
    linkcolor={blue!80!black},
    citecolor={blue!80!black},
    urlcolor={blue!80!black}
}

\DeclareSymbolFontAlphabet{\mathbb}{AMSb}
\DeclareSymbolFontAlphabet{\mathbbl}{bbold}

\usepackage{thmtools,thm-restate}

\usepackage[capitalise]{cleveref}
\usepackage{dsfont}

\usepackage{tikz}
\usetikzlibrary{positioning}
\usepackage{makecell}
\usetikzlibrary{fit}
\usetikzlibrary{shapes.geometric}

\def\X{{\mathcal X}}
\def\H{{\mathcal H}}

\def\Y{{\mathcal Y}}

\def\B{{\mathcal B}}
\def\E{{\mathbb E}}

\def\A{{\mathcal A}}

\newcommand{\D}{\mathcal{D}}

\newcommand{\bzero}{\ensuremath{\mathbf 0}}

\renewcommand\L{\mathcal{L}}
\newcommand{\F}{\mathcal{F}}

\newcommand{\C}{C}

\def\x{\mathbf{x}}

\def\w{\mathbf{w}}

\def\bz{\mathbf{z}}
\def\bp{\mathbf{p}}
\def\bq{\mathbf{q}}

\def\bone{\mathbf{1}}

\newcommand{\ignore}[1]{}

\theoremstyle{plain}
\newtheorem{theorem}{Theorem}
\newtheorem{lemma}[theorem]{Lemma}

\newtheorem{proposition}[theorem]{Proposition}

\newtheorem{fact}[theorem]{Fact}

\newtheorem*{theorem*}{Theorem}
\newtheorem*{lemma*}{Lemma}
\newtheorem*{corollary*}{Corollary}
\newtheorem*{proposition*}{Proposition}
\newtheorem*{claim*}{Claim}
\newtheorem*{fact*}{Fact}
\newtheorem*{observation*}{Observation}
\newtheorem*{assumption*}{Assumption}

\theoremstyle{definition}
\newtheorem{definition}[theorem]{Definition}

\newtheorem*{definition*}{Definition}
\newtheorem*{remark*}{Remark}
\newtheorem*{example*}{Example}

\theoremstyle{plain}

\DeclareMathAlphabet{\mathbfsf}{\encodingdefault}{\sfdefault}{bx}{n}

\DeclareMathOperator*{\argmin}{arg\,min}

\let\Pr\relax
\DeclareMathOperator{\Pr}{\mathbb{P}}

\newcommand{\lrset}[1]{\left\{#1\right\}}

\DeclarePairedDelimiter{\abs}{|}{|}
\DeclarePairedDelimiter{\norm}{\|}{\|}
\DeclarePairedDelimiter{\ceil}{\lceil}{\rceil}

\newcommand{\ind}[1]{\mathbb{I}\!\lrset{#1}}

\newcommand{\eps}{\varepsilon}

\renewcommand{\leq}{~\le~}

\let\oldtfrac\tfrac
\renewcommand{\tfrac}[2]{\smash{\oldtfrac{#1}{#2}}}

\let\nablaold\nabla
\renewcommand{\nabla}{\nablaold\mkern-2.5mu}

\newtheorem{atheorem}{Theorem}

\usepackage{bbm}
\newcommand{\cost}{w}
\newcommand{\costVec}{\boldsymbol{w}}
\newcommand{\bbz}{\boldsymbol{z}}
\newcommand{\balpha}{\boldsymbol{\alpha}}
\newcommand{\bbe}{\boldsymbol{e}}

\def\argmin{\mathop{\rm arg\, min}}

\renewcommand{\eps}{\epsilon}
\newcommand{\bx}{{\boldsymbol{x}}}

\usepackage{todonotes}
\definecolor{PalePurp}{rgb}{0.66,0.57,0.66}

\DeclareSymbolFontAlphabet{\mathbb}{AMSb}

\newcommand{\NN}{\mathbb{N}}
\newcommand{\RR}{\mathbb{R}}
\newcommand{\val}{\operatorname{V}}
\newcommand{\RRp}{\RR_{\ge 0}}

\newcommand{\listSize}{s}

\newcommand{\threshIndex}{n}
\newcommand{\excl}{\operatorname{av}}

%% file: 1-intro.tex
\newpage
\section{Introduction}\label{sec:intro}
In many machine learning applications, different types of mistakes may have very different consequences, making it crucial to consider the costs associated with them. For example, in medical diagnostics, failing to detect a serious illness (a false negative) can have life-threatening implications, whereas incorrectly diagnosing a healthy person as ill (a false positive) mostly leads to unnecessary stress and medical expenses.
This disparity in error costs is not limited to binary decisions. For example, when recommending movies to a viewer with preferences ``romance over action over horror'', misclassifying a romance film as ``horror'' is probably worse than misclassifying it as ``action''.
Besides weighting different kinds of mispredictions, one may even want to treat different kinds of mispredictions separately.
That is, instead of a cost-sensitive criterion, one may use a multi-objective criterion, specifying acceptable rates for different types of mispredictions. For example, one may find acceptable a false positive rate of (say) $10\%$ only if simultaneously the false negative rate is at most (say) $1\%$.

Despite the importance of misclassification costs in applications, the theoretical understanding of this setting is lacking. A glaring example, which motivates this work, is boosting. Broadly speaking, a boosting algorithm is a procedure that aggregates several \emph{weak learners} (whose accuracy is only marginally better than a random guess) into a single \emph{strong learner} (whose accuracy can be made as high as desired). Although this is a fundamental and well-studied machine learning technique, a theory of boosting accounting for cost-sensitive or multi-objective losses is missing, even in the simplest setting of binary classification.\footnote{There are many works on adapting AdaBoost to cost-sensitive learning, but they do not address the fundamental question of identifying the minimal assumption on the cost-sensitive function which guarantees boosting. See more in Appendix~\ref{app:cost-sensitive}.} In fact, if one can assign different costs to different kinds of mistakes, then even the meaning of ``marginaly better than a random guess'' is not immediately clear; let alone the question of whether one can boost a weak learner to a strong learner, or what precisely ``boosting'' means. The present work addresses those challenges, providing a generalized theory of boosting which unifies different types of weak learning guarantees, including cost-sensitive and multi-objective ones, and extends the standard algorithmic boosting framework beyond the current state of the art. The fundamental question that we pose is:
\begin{center}
     \emph{Which cost-sensitive and/or multi-objective learners can be boosted? And how?} 
\end{center}
In classical boosting theory for binary classification, a sharp transition occurs at a weak learner's error threshold of \( 1/2 \): if the weak learner is guaranteed to output hypotheses with an error rate below \( 1/2 \), then it can be boosted to a strong learner with arbitrarily small error, for instance by AdaBoost \citep{Freund97decision}. Thus, a weak learning guarantee below \( 1/2 \) implies strong learning. On the other hand, a guarantee of \( 1/2 \) is trivial, as it can be achieved by tossing a fair coin---see \cite{schapire2012boosting}. Therefore, guarantees above \( 1/2 \) are not boostable.

We investigate whether similar transitions exist for arbitrary cost-sensitive losses. A cost-sensitive loss \( \cost \) specifies the penalty $\cost(i,j)$ for predicting $i$ when the true label is $j$, and can penalize prediction errors unequally (e.g., in binary classification, it may penalize false negatives more than false positives). Suppose we have access to a weak learner that outputs hypotheses with a cost-sensitive loss of at most \( z > 0 \) under \( \cost \). For which values of \( z \) does this imply strong learning, so that the weak learner can be boosted to achieve arbitrarily small cost-sensitive loss according to \( \cost \)? Which values of \( z \) are trivial, meaning they can always be achieved? Are there intermediate values of $z$ that do not imply strong learning but are still non-trivial? 

A similar question arises for multi-objective losses. %
A multi-objective loss is given by a vector \( \costVec=(\cost_1, \cost_2, \ldots, \cost_r) \) where each $\cost_i$ is a cost-sensitive loss as described above. For instance, in binary classification, a natural choice would be \( \costVec=(\cost_n, \cost_p) \), where \( \cost_n \) measures false negatives and \( \cost_p \) false positives. 
Suppose again we have access to a weak learner that outputs hypotheses with loss at most \( z_i \ge 0 \) under \( \cost_i \) simultaneously for every $i$, forming a vector of guarantees \( \bz = (z_1, \ldots, z_r) \). Which $\bz$ are trivial, in that they can always be achieved? Which $\bz$ are boostable, allowing for a strong learner that achieves arbitrarily small error simultaneously for all of the losses \( \cost_i \)? And are there intermediate vectors $\bz$ that fall between trivial and boostable?

We address these questions by introducing a new perspective on random guessing, framed as either a zero-sum game or a vector-payoff game (known as a Blackwell approachability game). This game-theoretic approach applies to both cost-sensitive and multi-objective learning, leading to a complete characterization of boostability in these cases. We then extend these techniques to the multiclass setting, where the boostability landscape becomes significantly more complex. While this perspective complements existing views of boosting as a zero-sum game, prior methods are not suited to the  settings we examine here. The new tools introduced in this work 
effectively handle a broad learning framework, 
establishing a unified and comprehensive
theory of generalized boosting.

In particular, we provide extensive answers to the above questions, as follows:
\begin{itemize}
    \item \textit{Cost-sensitive losses}.
    For binary classification, we establish a crisp dichotomy: each guarantee \( z \ge 0 \) is either trivial or boostable (\Cref{thm:intro_binary_boost}). We show that this transition occurs at a critical threshold given by the value of a zero-sum game defined by \( \cost \) itself. In the multiclass setting, the dichotomy expands into a hierarchy of guarantees, ranging from fully boostable to trivial. Here, we show that there exist multiple thresholds \( 0 < v_1(\cost) < v_2(\cost) < \ldots < v_\tau(\cost) \), where \( \tau \) depends on the cost function $\cost$, and each guarantee \( z \in (v_i, v_{i+1}) \) can be boosted down to \( v_i \) (\Cref{thm:main_multiclass}). This generalizes the binary case, in which \( \tau = 1 \).

    \item \textit{Multi-objective losses}. For binary classification, we again show a clean dichotomy; however, the threshold now takes a higher-dimensional form, becoming a surface that separates trivial guarantees from boostable ones (\Cref{thm:binary_MO_boost}). \Cref{fig:boostability-thresholds} illustrates this surface for a loss vector \( \costVec \) representing false positive and false negative errors. In the multiclass setting, things become more complex: here, we show how to boost non-trivial guarantees to list-learners  (\Cref{thm:multiclass_MO_boost}), but a complete taxonomy of boostable guarantees remains elusive and is left as an open question.
    
    \item \textit{An equivalence between cost-sensitive and multi-objective losses.} We establish and exploit an equivalence between multi-objective and cost-sensitive losses that may be of independent interest (\Cref{thm:intro_general_duality}). Given a loss vector \( \costVec = (\cost_1, \ldots, \cost_r) \), consider a weak learner that outputs hypotheses with loss at most \( z_i \) for each \(\cost_i\). By linearity of expectation, it follows that for any convex combination \( \cost_{\balpha} = \sum \alpha_i \cost_i \), the weak learner’s loss with respect to \( \cost_{\balpha} \) does not exceed \( \sum \alpha_i z_i \). We prove that the converse also holds: if for each such \( \cost_{\balpha} \) there exists a weak learner with loss at most \( \sum \alpha_i z_i \), then we can efficiently aggregate these learners into one that achieves loss at most \( z_i \) simultaneously for every \(\cost_i\).
    Interestingly, a geometric perspective of this result reveals a connection to Blackwell's Approachability Theorem~\citep{blackwell1956analog} and to scalarization methods in multi-objective optimization.
\end{itemize}

\paragraph{Organization of the manuscript.} \Cref{sec:results} gives a detailed walkthrough of all our main results, their significance, and the underlying intuition. \Cref{sec:binary} addresses the binary classification case, in both the cost-sensitive and multi-objective flavors. \Cref{sec:duality} presents our equivalence connecting cost-sensitive learners with multi-objective learners. Finally, \Cref{sec:multiclass} considers the multiclass case.

%% file: 2-results.tex
\section{Main Results}
\label{sec:results}
This section outlines the problem setting and the key questions we investigate, and provides an overview of our main results. We begin from the basic technical setup. We consider the standard PAC (Probably Approximately Correct) setting \citep{Valiant84}, which models learning a concept class $\F \subseteq \Y^\X$ for a domain $\X$ and a label space $\Y=[k]$ with $k \ge 2$. We define a \emph{cost function}, or simply \emph{cost}, to be any function $\cost:\Y^2 \to [0,1]$ that satisfies $\cost(i,i)=0$ for all $i \in \Y$;\footnote{Although the range of the cost is $[0,1]$ our results can be easily extended to arbitrary costs $\cost:\Y^2 \to \mathbb{R}_{\ge 0}$.} the value $\cost(i,j)$ should be thought of as the penalty incurred by predicting $i$ when the true label is~$j$. Note that $\cost$ can be seen as a $k \times k$ matrix indexed by $\Y$. For instance, for $\Y=\{-1,+1\}$, a valid cost is $\cost=\bigl(\begin{smallmatrix} 0 & 1\\ 4 & 0 \end{smallmatrix}\bigr)$; this means that the cost of a false positive is $\cost(+1,-1) = 4$, and that of a false negative is $\cost(-1,+1) = 1$.  
For conciseness, in the binary setting we let $w_- = \cost(-1,+1)$ and $w_+ = \cost(+1,-1)$; and by some overloading of notation we denote by $\cost$ both the matrix $\bigl(\begin{smallmatrix} 0  & w_-\\ w_+ & 0 \end{smallmatrix}\bigr)$ and the vector $(w_+,w_-)$. 
For a generic cost $\cost: \Y^2 \rightarrow [0,1]$, a target function $f \in \F$, and a distribution $\D$ over $\X$, the cost-sensitive {loss} of a predictor $h : \X \to \Y$ is:
\begin{equation} 
 L_{\D}^{\cost}(h) \triangleq \E_{x \sim \D}\Bigl[\cost(h(x),f(x))\Bigr]
\end{equation} 
For example, when $\cost(i, j) = \ind{i \neq j}$, then $L_{\D}^{\cost}(h)$ simply corresponds to $\Pr_{x \sim \D}(h(x) \neq f(x))$, the standard 0-1  loss used in binary classification.
For convenience we assume throughout the paper that $\norm{w}_{\infty} \le 1$, though all our results apply more broadly.

\medskip

\noindent We begin by presenting our contributions for the binary setting (\Cref{sub:intro_bin}), followed by their extension to the general multiclass case (\Cref{subsec:main_results_multiclass}).

\subsection{Binary setting}\label{sub:intro_bin}
We begin from the binary setting, that is, the case $\Y=\{-1,+1\}$. The starting point is the definition of a suitable notion of ``weak learner'' for an arbitrary cost function $w$. 
\begin{definition}[{{\bf $(\cost, z)$-learner}}]\label{def:w_z_learner}
Let $\cost: \Y^2 \rightarrow [0,1]$ be a cost function 
and let $z \ge 0$. 
An algorithm~$\A$ is a $(\cost, z)$-learner for $\F \subseteq \Y^\X$ if there is a function $m_0: (0,1)^2 \rightarrow \mathbb{N}$ such that for every $f \in \F$, every distribution $\D$ over $\X$, and every $\epsilon,\delta \in (0,1)$ the following claim holds. If $S$ is a sample of $m_0(\epsilon, \delta)$  examples drawn i.i.d.\ from $\D$ and labeled by $f$, then $\A(S)$ returns a predictor $h: \X \rightarrow \Y$ such that $L_{\D}^{\cost}(h) \le z + \epsilon$ with probability at least $1-\delta$.
\end{definition}

\begin{figure}[!t]
    \begin{floatrow}
    {%
    \renewcommand{\arraystretch}{1.1}
    \begin{tabularx}{0.7\textwidth}{ 
       >{\centering\arraybackslash}X 
      | >{\centering\arraybackslash}X 
      | >{\centering\arraybackslash}X  }
     {} & {\bf Binary} & {\bf Multiclass} \\
     \hline 
     {\bf \quad 0-1 loss} & $1/2$ & $\frac{1}{2}, \frac{2}{3}, \frac{3}{4},\dots, \frac{k-1}{k}$ \vspace{.05cm} \\
     \hline
     {\bf \quad \ cost $\cost$} & $V(\cost) = \frac{\cost_+\cost_-}{\cost_+ +\cost_-}$ & { $v_{{2}}(\cost) \le\dots\le v_{{\tau}}(\cost)$}\vspace{.05cm}  \\
    \hline
     {\bf  {\small multi-objective} { $\costVec = (\cost_1,\dots,\cost_r)$}} &  {\small coin-attainable boundary of $C(\costVec)$} &  {\small $J$-dice-attainable boundaries of $D_J(\costVec)$}
    \end{tabularx}
    }
    \ffigbox{%
     \hspace{-3cm} \vspace{-1.6cm}
     \includegraphics[width=0.27\textwidth]{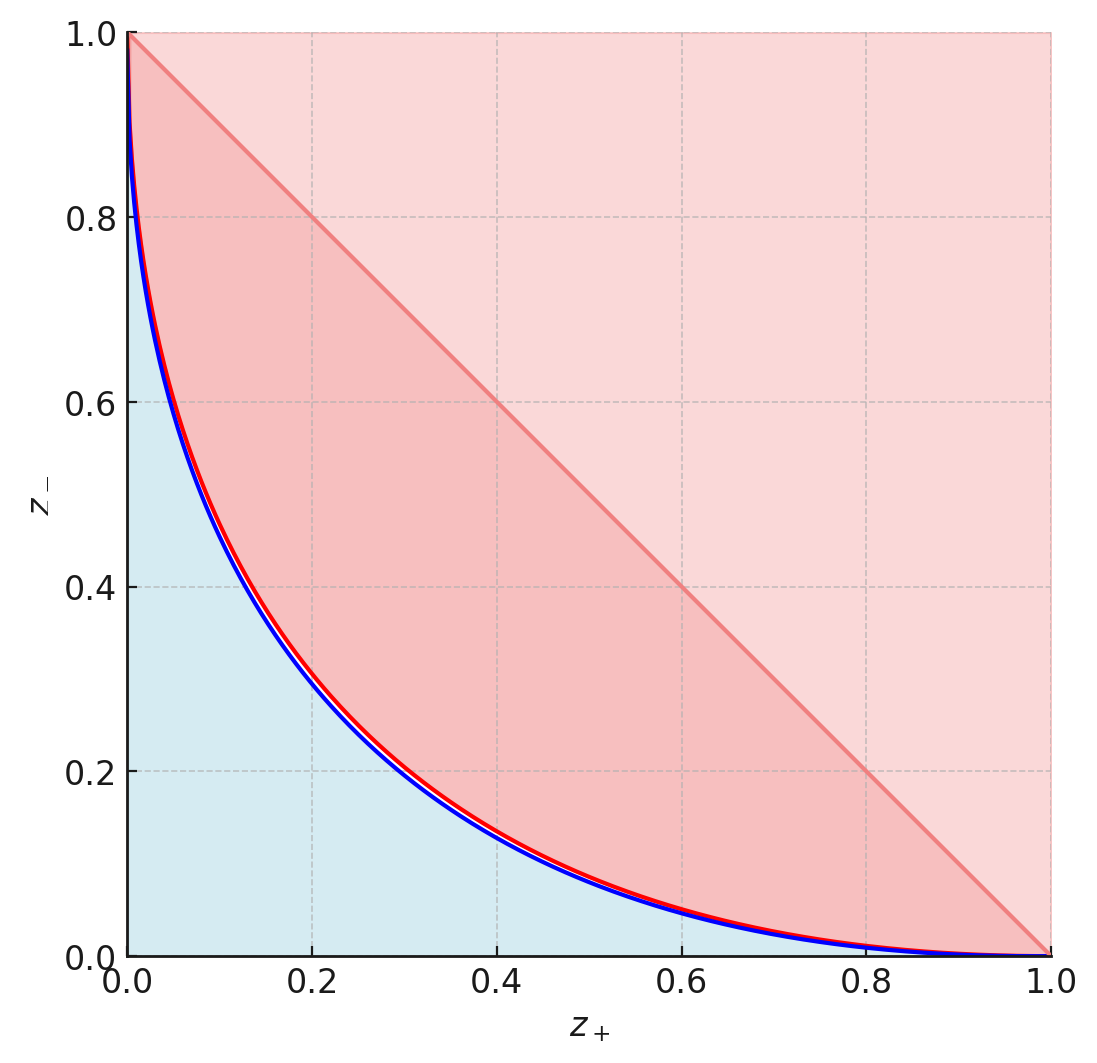}
    }{}
    \end{floatrow}
    \caption{{\bf Boostability thresholds.} \underline{Binary}. 
    For classic 0-1 loss binary boosting, it is well-known that the boostability threshold is $1/2$: any value below it can be boosted, while any value above it is trivially attainable by non-boostable learners. For any cost $\cost$, the boostability threshold is $\val(\cost)$ (see \Cref{eq:val_of_game}, \Cref{thm:intro_binary_boost}).  For the multi-objective loss, the threshold is determined by the boundary of the coin-attainable region, denoted $C(\costVec)$ (see \Cref{def:coin_attainability}, \Cref{thm:binary_MO_boost}), as illustrated in the plot on the right; each point in the plane corresponds to false-positive and false-negative errors $(z_+,z_-)$. The two colored regions in the plot correspond to (a) coin-attainable point $C(\costVec)$ (in red) and (b) boostable points $[0,1]^2\setminus C(\costVec)$ (in blue). 
    See below \Cref{thm:binary_MO_boost} for further discussion.
    \underline{Multiclass}. A similar pattern holds for multiclass boosting. For 0-1 loss, boostability is known to be determined by $k-1$ thresholds \citep{Brukhim23simple}. 
    For any cost $\cost$, the boostability thresholds are $v_{{n}}(\cost)$ (see \Cref{eq:multiclass_critical_thresholds}). For the multi-objective loss, thresholds are determined by the boundaries of dice-attainable regions $D_J(\costVec)$ (see \Cref{subsec:main_results_multiclass} for further details).
    }\label{fig:boostability-thresholds}
\end{figure}

We study the question of whether a $(\cost, z)$-learner can be boosted. Broadly speaking, a learner~$\A$ is {\it boostable} if there exists a learning algorithm $\B$ that can achieve arbitrarily small loss by aggregating a small set of predictors obtained from $\A$. More precisely, given black-box access to $\A$ and a labeled sample $S$, algorithm $\B$ invokes $\A$ on subsamples $S_1,\ldots,S_T$ of $S$ to produce weak predictors~$h_1, \dots, h_T$. It then outputs a predictor $\bar{h}(x) = g(h_1(x), \dots, h_T(x))$ using some aggregation function $g: \Y^T \rightarrow \Y$.  For instance, in classic binary boosting the function $g$ is usually a weighted majority vote. In this work, we consider arbitrary such aggregations.  The goal is to ensure that the loss of the aggregate predictor, $L_\D^{\cost}(\bar{h})$, goes to zero with $T$, resulting in a $(\cost,0)$-learner. Our guiding question can then be stated as:
\begin{center}
    {\it Can we boost a $(\cost,z)$-learner to a $(\cost,0)$-learner?}
\end{center}
In order to develop some intuition, let us consider again the case when $\cost$ yields the standard 0-1 loss. In that case, \Cref{def:w_z_learner} boils down to the standard notion of a weak PAC learner.\footnote{Notice that this definition is slightly more general than the classic weak PAC learning definition, in which $z = \nicefrac{1}{2} - \gamma$ and $\epsilon=0$, yet we consider learners which are allowed to be arbitrarily close to $z$.} Then, classic boosting theory states that every $(\cost,z)$-learner with $z<\nicefrac{1}{2}$ can be boosted to a $(\cost,0)$-learner; and it is easy to see that a loss of $\nicefrac{1}{2}$ is always achievable, simply by predicting according to a fair coin toss. Thus, the value $\nicefrac{1}{2}$ yields a sharp dichotomy: every $z < \nicefrac{1}{2}$ can be boosted and drawn to $0$, while every $z \ge \nicefrac{1}{2}$ is trivially achievable and cannot be brought below $\nicefrac{1}{2}$.

Can this classic dichotomy between ``boostable'' and ``trivial'' be extended to accommodate arbitrary cost functions? It turns out that this can be done if one uses as trivial predictors {\it biased} coins. 
Indeed, we show that, by taking such random guessing strategies into account, one can identify a general boostability threshold for all cost functions $\cost$.
However, in contrast to the 0-1 loss case, this critical threshold between the boostable and non-boostable guarantees is no longer $\nicefrac{1}{2}$; instead, it is a function of the cost $\cost$. More precisely, the threshold is determined by the outcome of a simple two-player zero-sum game, as we describe next.  

The game involves a minimizing player (the predictor) and a maximizing player (the environment). The minimizing player selects a distribution $\bp$ over $\Y=\{-1,+1\}$; similarly, the maximizing player selects a distribution $\bq$ over $\Y=\{-1,+1\}$. The payoff matrix of the game is $\cost$. The overall cost paid by the predictor is then:
\begin{equation}\label{eq:zero_sum_game_def}
    \cost(\bp,\bq) \triangleq \E_{i \sim \bp} \E_{j \sim \bq} \  \cost (i,j) \,.
\end{equation}
Following standard game theory, the \emph{value of the game} is the quantity
\begin{align}\label{eq:val_of_game}
    \val(\cost) \triangleq \min_{\bp \in \Delta_\Y} \max_{\bq \in \Delta_\Y} \cost(\bp,\bq)\,,
\end{align}
where $\Delta_\Y$ is the set of all distributions over $\Y$.
In words, $\val(\cost)$ is the smallest loss that the predictor can ensure by tossing a biased coin (i.e., $\bp$) without knowing anything about the true distribution of the labels (i.e., $\bq$). 
Now consider a value $z \ge 0$. It is easy to see that, if $z \ge \val(\cost)$, then there exists a universal $(\cost, z)$-learner---one that works for all $\F$ and all distributions $\D$ over $\X$: this is the learner that, ignoring the input sample, returns the randomized predictor $h_{\bp}$ whose outcome $h(x)$ is distributed according to $\bp$ independently for every $x \in \X$. Indeed, by \Cref{eq:val_of_game} the loss of $h_{\bp}$ is at most $\val(\cost)$ regardless of $f \in \F$ and of the distribution $\D$. Formally, we define:
\begin{definition}[Random guess]\label{def:random-guess}
    A randomized hypothesis is called a \emph{random guess} if its prediction $y \in \Y$ is independent of the input point $x \in \X$. That is, there exists a probability distribution $\bp \in \Delta_\Y$ such that $h(x)\sim \bp$ for every $x \in \X$.
\end{definition}
\noindent We prove that the non-boostability condition $z \ge \val(\cost)$ is tight. That is, we prove that every $(\cost,z)$-learner with $z<\val(\cost)$ \emph{is} boostable to a $(\cost,0)$-learner. Thus, the value of the game $\val(\cost)$ given by \Cref{eq:val_of_game} is precisely the threshold for boostability. Formally:
\begin{atheorem}[Cost-sensitive boosting, binary case]\label{thm:intro_binary_boost}
Let $\Y=\{-1,+1\}$. Let $\cost = (w_+, w_-) \in (0,1]^2$ be a cost. Then, for all $z \ge 0$, exactly one of the following holds.
\begin{itemize}[leftmargin=.6cm]\itemsep0pt
    \item  {\bf\small $(\cost, z)$ is boostable}: for every $\F \subseteq \Y^\X$, every $(\cost,z)$-learner is boostable  to a $(\cost,0)$-learner.
    \item {\bf\small $(\cost, z)$ is trivial}: there exists a random guess $h$ such that the learner that always outputs $h$ is a $(\cost, z)$-learner for every $\F \subseteq \Y^\X$.
\end{itemize}
Moreover, $(\cost,z)$ is boostable if and only if $z < V(\cost)$, where $\val(\cost) = \frac{\cost_+ \cost_-}{\cost_+ + \cost_-}$.
\end{atheorem}

\noindent The proof of \Cref{thm:intro_binary_boost} is given in \Cref{sec:binary}.
In a nutshell, \Cref{thm:intro_binary_boost} says that anything that beats the weakest possible learner (a coin) is as powerful as a the strongest possible learner; between these two extremes there is no middle ground.  Remarkably, the proof of \Cref{thm:intro_binary_boost} is simple and yet it relies on {\it three} distinct applications of von Neumann’s Minimax Theorem! \citep{neumann1928theorie}.  The first application, which also appears in classical binary boosting, is used to aggregate a distribution over the weak learners. The other two applications are unique to the cost-sensitive setting: one arises in the analysis of the boosting algorithm (first item of \Cref{thm:intro_binary_boost}), and the other one in defining the trivial learner (second item of \Cref{thm:intro_binary_boost}). These last two applications are both based on the zero-sum game defined above. We also note that the first item of \Cref{thm:intro_binary_boost} is obtained  constructively by an efficient boosting algorithm, as we detail in \Cref{sec:binary}.

\subsubsection{Multi-objective losses}
\begin{figure}[t]
    \centering
    {{\includegraphics[width=45mm]{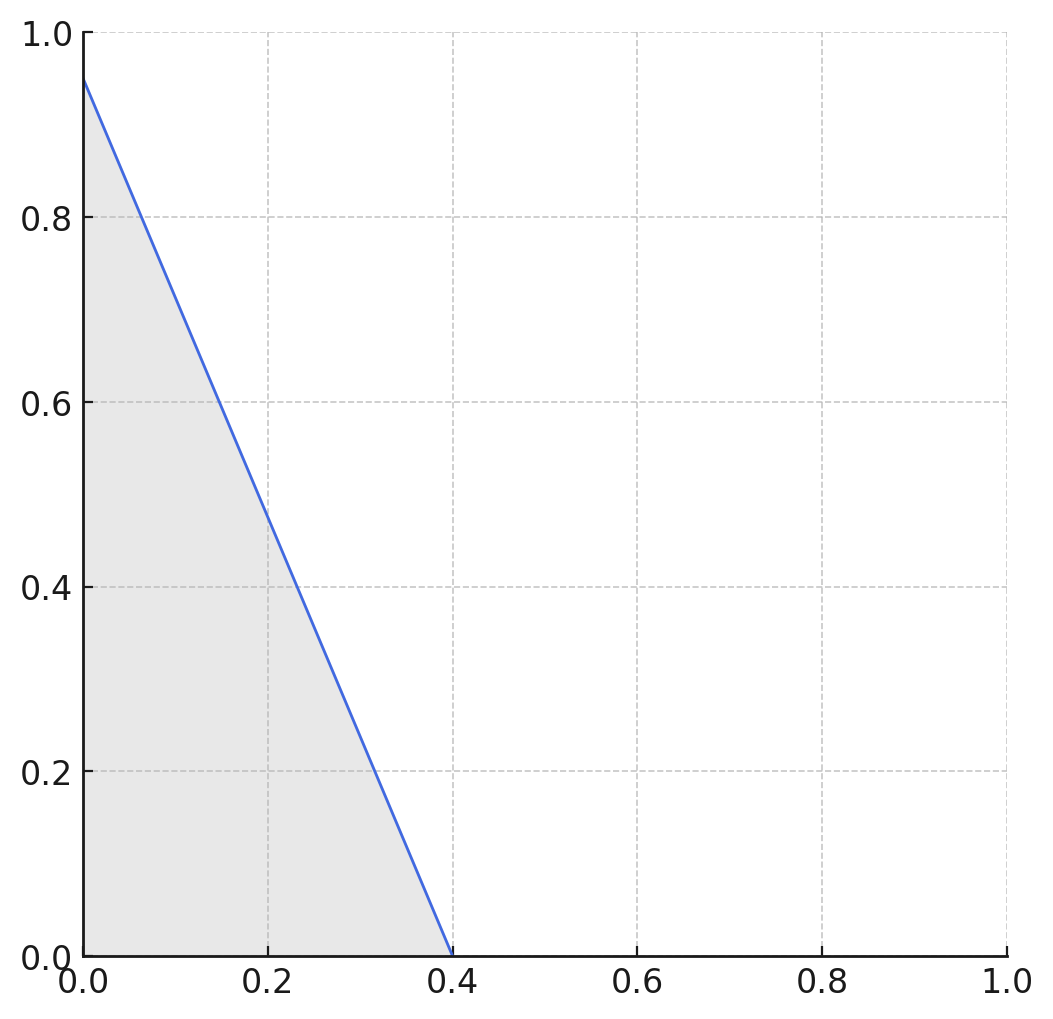} }}%
    \quad {{\includegraphics[width=45mm]{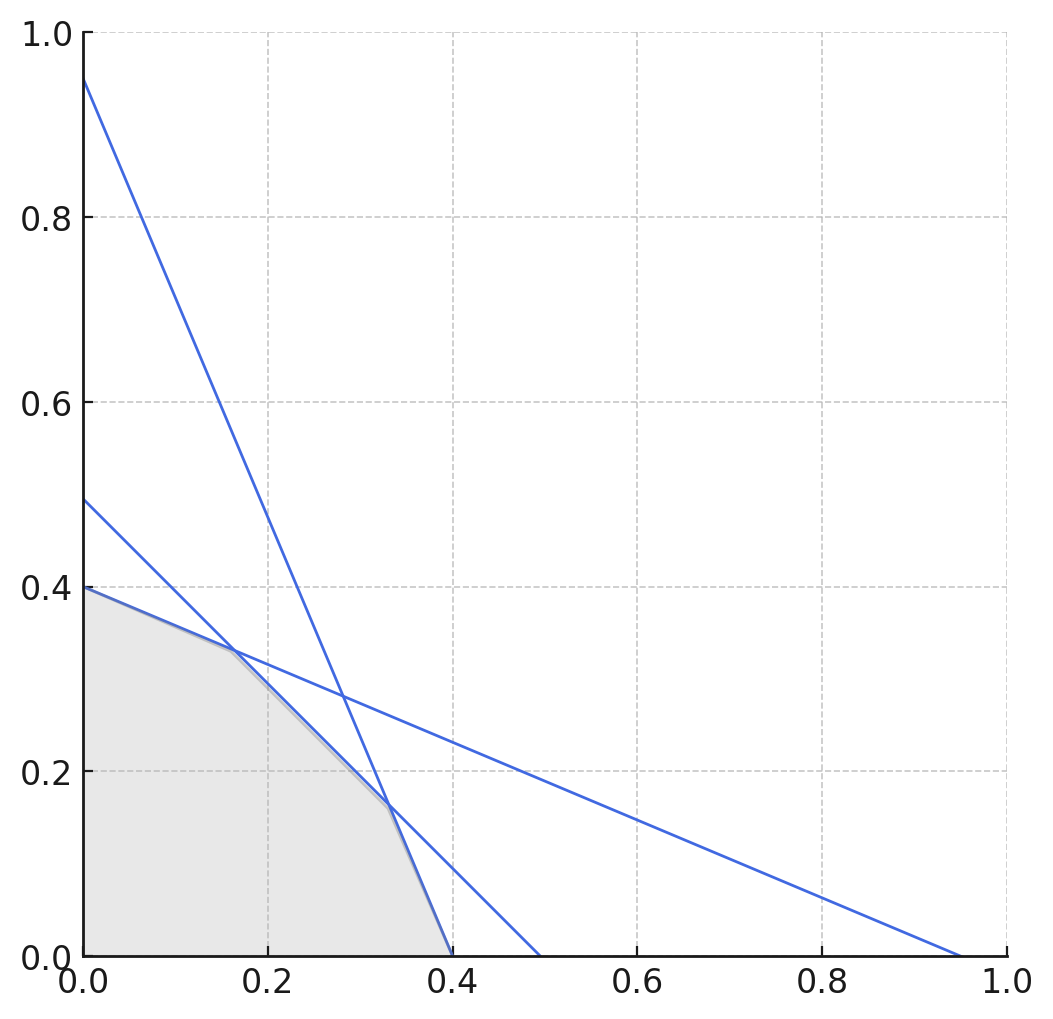} }}%
    \hfill {{\includegraphics[width=45mm]{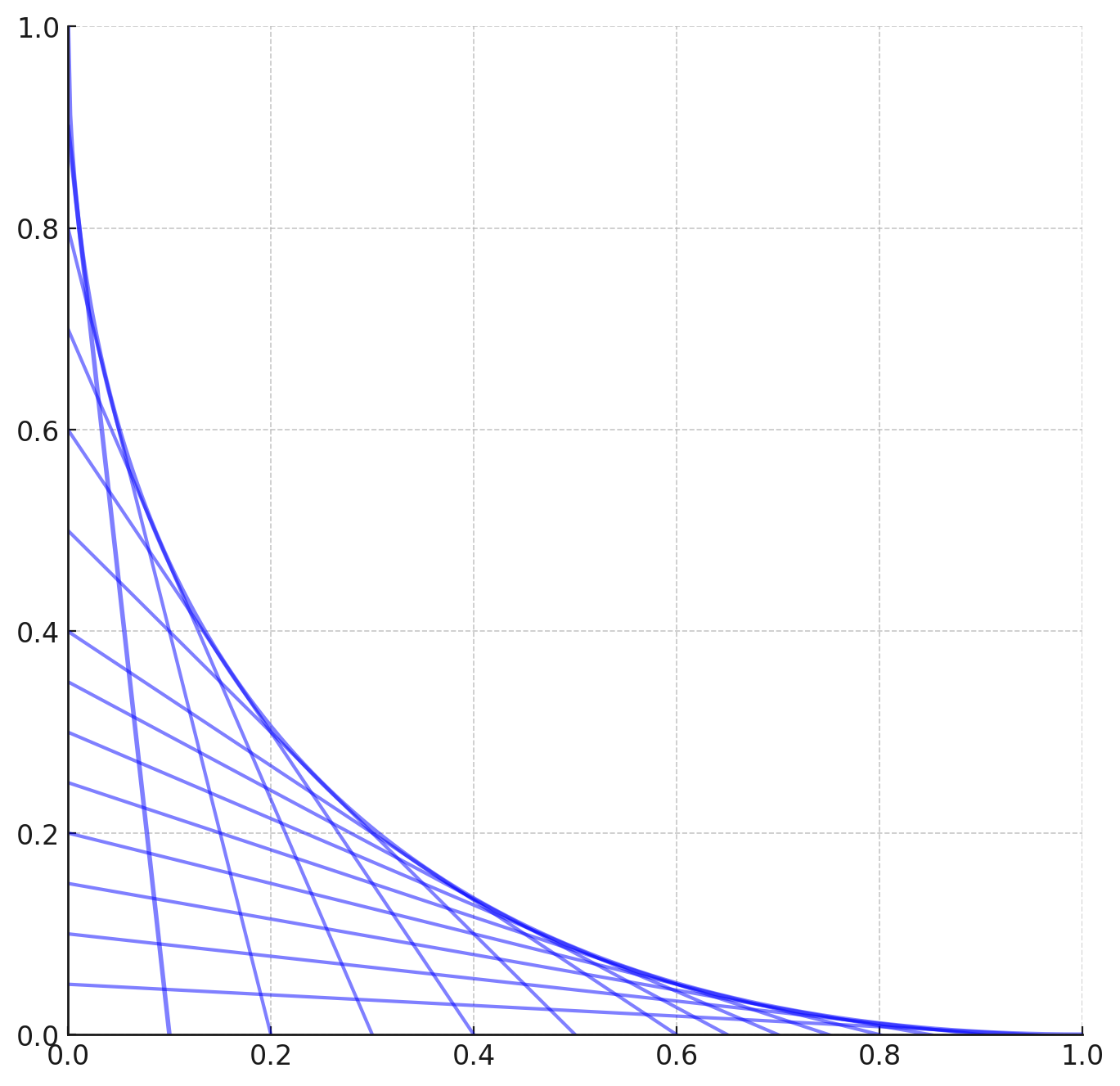} }}%
    \caption{ In all plots, each point $e = (e_+,e_-)$ in the plane corresponds to false-positive and false-negative errors. {\small(Left) {\bf Cost-sensitive vs.\ multi-objective}}. The leftmost figure corresponds to a cost-sensitive guarantee $(\cost,z)$, where the blue line is given by $\langle e, w\rangle = z$. The shaded area is the feasible region of points $e$ satisfying the  guarantee. 
    The second figure corresponds to a multi-objective guarantee $(\costVec,\bbz)$, where $r=3$ corresponds to 3 different lines, each of the form  $\langle e, w_i\rangle = z_i$. The shaded area corresponds to all points satisfying all guarantees, i.e, attaining  $(\costVec,\bbz)$. {\small (Right) {\bf Envelope of the coin-attainable region}}. The rightmost figure presents many different lines; each line correspond to a guarantee $(\cost,\val(\cost))$ and is of the form $\langle e, w\rangle = \val(\cost)$. %
    Then, the coin-attainable boundary curve in the case of false-positive and false-negative costs   
    (i.e., $\costVec = (w_p,w_n)$), is given by $\sqrt{e_+} + \sqrt{e_-} = 1$ (as shown in \Cref{sec:binary}). Furthermore, that boundary of the coin-attainable points is the same curve obtained by the envelope of the different lines. This ties between the value $\val(\cost)$   to the coin-attainable area $C(\costVec)$ 
    (see below \Cref{thm:binary_MO_boost} for further discussion).
    }%
    \label{fig:MO_vs_CS}%
\end{figure}
We turn our attention to \emph{multi-objective} losses. A multi-objective loss is specified by a multi-objective cost, that is, a vector $\costVec = (\cost_1,\ldots,\cost_r)$ where $\cost_i:\Y^2\to[0,1]^2$ is a cost for every $i=1,\ldots,r$. This allows one to express several criteria by which a learner is evaluated; for example, it allows us to measure separately the false positive rate and the false negative rate, giving a more fine-grained control over the error bounds.
The guarantees of a learner can then be expressed by bounding $L_{\D}^{\cost_i}(h)$ by some $z_i \ge 0$ simultaneously for each $i =1,\ldots, r$.
This leads to the following generalization of \Cref{def:w_z_learner}.
\begin{definition}[{{\bf $(\costVec, \bbz)$-learner}}]\label{def:vec_z_learner}
Let $r \in \mathbb{N}$, let $\costVec = (\cost_1,\ldots,\cost_r)$ where each $\cost_i:\Y^2\to[0,1]^2$ is a cost function, and let $\bbz \in [0,1]^r$. 
An algorithm $\A$ is a $(\costVec, \bbz)$-learner for $\F \subseteq \Y^\X$ if there is a function $m_0: (0,1)^2 \rightarrow \mathbb{N}$ such that for every $f \in \F$, every distribution $\D$ over $\X$, and every $\epsilon,\delta \in (0,1)$ what follows holds. If $S$ is a sample of $m_0(\epsilon, \delta)$  examples drawn i.i.d.\ from $\D$ and labeled by $f$, then $\A(S)$ returns a predictor $h: \X \rightarrow \Y$ such that with probability at least $1-\delta$: 
$$\forall \ i = 1, \ldots, r,   \qquad  
  L^{\cost_i}_\D(h) \le z_{i} + \epsilon \;.
  $$
\end{definition}
\noindent 
Consider for example $\costVec = (\cost_p, \cost_n)$ with $\cost_p=\bigl(\begin{smallmatrix} 0 & 0\\1  & 0 \end{smallmatrix}\bigr)$ and $\cost_n=\bigl(\begin{smallmatrix} 0 & 1\\0 & 0 \end{smallmatrix}\bigr)$. Then $\cost_p$ counts the false positives, $\cost_n$ the false negatives. Thus a $(\costVec,\bz)$-learner for, say, $\bz=(0.1,0.4)$ ensures simultaneously a false-positive rate arbitrarily close to $0.1$ and false-negative arbitrarily close to $0.4$. See \Cref{fig:MO_vs_CS} for illustrative geometric examples.

Like for the cost-sensitive case, we address the question of which $(\costVec, \bbz)$-learners are boostable. In other words we ask: given $\costVec$, what is the right ``threshold'' for $\bz$? Equivalently, for which $\bz$ can we always boost a $(\costVec, \bbz)$-learner to a $(\costVec, \mathbf{0})$-learner, and for which $\bz$ do there exist $(\costVec, \bbz)$-learners that are not boostable?
It turns out that the answer is more nuanced than in the cost-sensitive case of \Cref{sub:intro_bin}, and yet we can prove the same ``all or nothing'' phenomenon of \Cref{thm:intro_binary_boost}, as depicted in \Cref{fig:boostability-thresholds}.\footnote{Note that the light-red area (upper triangle) is attainable by coins (distributions) that are oblivious to the true distribution marginals. For example, it is trivial to attain the point $(\frac{1}{2}, \frac{1}{2})$ by a fair coin, and the points  $(0, 1)$ or $(1, 0)$ by a ``degenerate coin'' that is simply a constant function which always predicts the same label.} %
In particular, every $\bz \in [0,1]^r$ is either entirely boostable, in the sense that every $(\costVec, \bbz)$-learner can be boosted to a $(\costVec, \mathbf{0})$-learner, or \emph{trivial}, in the sense that there exists a $(\costVec, \bbz)$-learner whose output is always a hypothesis that can be simulated by a (biased) random coin. 
To this end we introduce the notion of {\it coin attainability}. This is the multi-objective equivalent of the condition $z \ge \val(\cost)$ for $\cost$ being a scalar cost as in \Cref{sub:intro_bin}.
\begin{definition}[{Coin attainability}]\label{def:coin_attainability}
    Let  $\costVec = (\cost_1,\ldots,\cost_r)$ where $\cost_i: \Y^2 \rightarrow [0,1]$ is a cost function for every $i=1,\ldots,r$, and let $\bz \in [0,1]^r$. We say that $(\costVec, \bbz)$ is \emph{coin-attainable} if
    \begin{equation}
        \forall \bq \in \Delta_\Y, \ \exists \bp \in \Delta_\Y, \ \forall i = 1, \ldots, r, \qquad \cost_i(\bp,\bq) \le z_i \;.
    \end{equation}
    The \emph{coin-attainable region} of $\costVec$ is the set $\C(\costVec)$ of all $\bbz$ such that $(\costVec, \bbz)$ is coin-attainable.
\end{definition}
\noindent 
It is not hard to see that, if $\bz$ is coin-attainable, there exists a universal trivial $(\costVec,\bz)$-learner for all $\F \subseteq \Y^\X$. Fix indeed any distribution $\D$ over $\X$ and any $f \in \F$. First, one can learn the marginal $\bq$ of $\D$ over $\Y$ within, say, total variation distance $\eps$; then, by the coin-attainability of $\bz$, one can return a hypothesis $h$ such that $h(x)\sim\bp$, where $\bp$ ensures $\cost_i(\bp,\bq) \le z_i + \eps$ for all $i=1,\ldots,r$. 
Formally, recalling the definition of random guess (\Cref{def:random-guess}), we have:
\begin{definition}[Trivial learner]\label{def:trivial-learner}
A learning algorithm $\A$ is \emph{trivial} if it only outputs random guesses.
\end{definition}
\noindent Thus, a trivial learner can (at best) learn the \emph{best} random guess and return it. Clearly, such a learner is in general not boostable to a $(\costVec, \bzero)$-learner. 
The main result of this subsection is:
\begin{atheorem}[Multi-objective boosting, binary case]\label{thm:binary_MO_boost}
Let $\Y = \{-1,1\}$,  let $\costVec = (\cost_1,\ldots,\cost_r)$ where $\cost_i: \Y^2 \rightarrow [0,1]$ is a cost, and let $\bbz \in [0,1]^r$. Then,   exactly one of the following holds. 
\begin{itemize}[leftmargin=.6cm]\itemsep0pt
    \item  {\bf\small $(\costVec, \bbz)$ is boostable}: for every $\F  \subseteq \Y^\X$, every
    $(\costVec, \bbz)$-learner is boostable to a $(\costVec,\mathbf{0})$-learner.
    \item  {\bf\small $(\costVec, \bbz)$ is trivial}: there exists a trivial learner that is a $(\costVec, \bbz)$-learner for all $\F \subseteq \Y^\X$.
\end{itemize}
Moreover, $(\costVec,\bz)$ is boostable if and only if $\bz \notin C(\costVec)$.
\end{atheorem}
The proof of \Cref{thm:binary_MO_boost} is given in \Cref{sec:binary}, and is based on a certain scalarization of the loss amenable to a reduction to \Cref{thm:intro_binary_boost}. Let us spend a few words on \Cref{thm:binary_MO_boost}. 
First, note how \Cref{def:coin_attainability} is reminiscent of the zero-sum game of \Cref{sub:intro_bin}. The crucial difference, however, is that now the maximizing player (the environment) plays first, and the minimizing player (the predictor) can then choose $\bp$ as a function of $\bq$. As shown in more detail in \Cref{subsec:binary:vector}, this is essential for proving the second item---that is, for coin-attainability to capture exactly the guarantees attainable by trivial learners. For the scalar-valued game of \Cref{sub:intro_bin}, instead, by von Neumann's Minimax Theorem the order of the players is immaterial to the value of the game~$\val(\cost)$.

Second, note that the multi-objective setting is more general than that of a (scalar-valued) zero-sum game, as here the payoff is vector-valued. This can be thought of as a Blackwell approachability game \citep{blackwell1956analog}, a generalization of zero-sum games to vectorial payoffs. Indeed, it turns out that there is a connection between these two notions, in the sense of Blackwell's theorem: we prove that a certain vectorial bound~$\bz$ can be attained (with respect to a vector-valued loss~$\costVec$) if an only if all the ``scalarizations'' of $\bz$ can be attained with respect to the corresponding scalarizations of $\costVec$. In fact, we can show that this is captured formally by an equivalence between cost-sensitive learners and multi-objective learners, as explained in \Cref{sub:intro_mo_equiv}.

\paragraph{Geometric interpretation of \Cref{thm:binary_MO_boost}.} A  multi-objective guarantee $\bz$ with respect to a multi-objective cost $\costVec$ can be viewed as a set of linear constraints, $\cost_i \le z_i$. This identifies a region that is convex and downward closed, see \Cref{fig:MO_vs_CS}.
This holds true even in the multiclass setting, when there are more than $2$ labels. Similarly, the set of all coin-attainable points $C(\costVec)$ for the false-negative/false-positive objective $\costVec = (\cost_p,\cost_n)$ can be shown to be convex and upward-closed\footnote{A set $C$ is upward-closed if for all $\bbz \in C$ and $\bbz' \in [0,1]^r$ such that $\bbz' \ge \bbz$ coordinate-wise, then $\bbz' \in C$. See Section \ref{sec:binary} for further details.}.
This is illustrated in \Cref{fig:boostability-thresholds}, as well as the set of all boostable points, given in the red and blue areas, respectively. 

As shown in \Cref{fig:MO_vs_CS}, the coin-attainable boundary curve in this case, is given by $\sqrt{e_+} + \sqrt{e_-} = 1$ (as shown in \Cref{sec:binary}), and it is exactly the same curve obtained by the envelope of the lines of the form $\langle e, w \rangle = V(w)$, for any cost $w: \Y^2 \rightarrow [0,1]$. This ties between the value $\val(\cost)$   to the coin-attainable area $C(\costVec)$. 

For a general pair $(\costVec, \bbz)$, one can characterize if it is coin-attainable by examining whether its associated convex feasible region is intersected with the set coin-attainable points as depicted in \Cref{fig:boostability-thresholds}. This intuition is also captured in \Cref{prop:intro_coin_duality} below.

\subsubsection{Cost-sensitive vs. multi-objective losses: a general equivalence}\label{sub:intro_mo_equiv}
We establish a general, formal connection between cost-sensitive and multi-objective learning. The starting point is the straightforward observation that, by definition, a $(\costVec, \bbz)$-learner is a $(w_i,z_i)$-learner for every $i=1,\ldots,r$. Does the converse hold, too? That is, if for each $i=1,\ldots,r$ there exist a $(w_i,z_i)$-learner $\A_i$, can we conclude that there exists a $(\costVec, \bbz)$-learner $\A$?
Perhaps surprisingly, we can show that this holds if we consider all \emph{convex combinations} of the scalar guarantees $(w_i,z_i)$. Formally, given a distribution $\balpha=(\alpha_1,\ldots,\alpha_r)\in \Delta_r$, let $w_{\balpha} = \sum_{i=1}^r \alpha_i w_i$ and $z_{\balpha} = \langle\balpha,\bbz\rangle$.
It is easy to see that the observation above continues to hold: a $(\costVec, \bbz)$-learner is a $(w_{\balpha},z_{\balpha})$-learner for every such $\balpha$. The next result shows that the converse holds, too: the existence of a $(w_{\balpha},z_{\balpha})$-learner for every $\balpha \in \Delta_{r}$ implies the existence of a $(\costVec, \bbz)$-learner.
\begin{theorem}[Equivalence of learners]\label{thm:intro_general_duality}
Let $\F \subseteq \Y^\X$, let $\costVec = (\cost_1,\ldots,\cost_r)$ where each $\cost_i:\Y^2 \to [0,1]$ is a cost function, and let $\bbz \in [0,1]^r$.  Then, the following are equivalent.
\begin{enumerate}\itemsep0pt
    \item There exists a $(\costVec, \bbz)$-learner for $\F$.
    \item For every $\balpha \in \Delta_r$ there exists a $\big(w_{\balpha}, z_{\balpha} \big)$-learner for $\F$. %
\end{enumerate} 
Moreover, the equivalence is obtained constructively by an efficient algorithm. 
\end{theorem}
\noindent 
We remark that \Cref{thm:intro_general_duality} holds in the multiclass case, too (i.e., for arbitrary sets $\Y$). Thus, the interplay between multi-objective and cost-sensitive learning is a general phenomenon, not limited to binary classifiers. 
The reduction from multi-objective to cost-sensitive in \Cref{thm:intro_general_duality} resembles the weighted sum scalarization method for multi-objective optimization \citep[Chapter 3]{ehrgott2005multicriteria}. However, it is unclear whether this similarity can be used to provide further insights in our problem.

In the same vein as \Cref{thm:intro_general_duality}, we prove an equivalence between trivial guarantees of cost-sensitive learners (see \Cref{thm:intro_binary_boost}) and trivial guarantees of multi-objective learners (see \Cref{thm:binary_MO_boost}).
\begin{theorem}[Equivalence of trivial guarantees]\label{prop:intro_coin_duality}
Let $\costVec = (\cost_1,\ldots,\cost_r)$ where each $\cost_i:\Y^2 \to [0,1]$ is a cost function, and let $\bbz \in [0,1]^r$.  Then the following are equivalent.
\begin{enumerate}\itemsep0pt
    \item $(\costVec, \bbz)$ is coin-attainable. 
    \item For every $\balpha \in \Delta_r$ it holds that $z_{\balpha}  \ge \val(w_{\balpha})$.
\end{enumerate} 
\end{theorem}
\noindent As 
\Cref{thm:intro_general_duality}, we note that
\Cref{prop:intro_coin_duality} also holds in the general multiclass case. The proof of \Cref{prop:intro_coin_duality} again uses von Neumann's Minimax Theorem, based on a two-player game that differs from those described earlier. In this game, the pure strategies of the maximizing player are the different elements of $[r]$, representing the various objectives of the cost function. 
The proofs of all equivalence theorems are given in Section \ref{sec:duality}.

\input{2.3-results-multi}

%% file: 2.3-results-multi.tex
\subsection{Multiclass setting}\label{subsec:main_results_multiclass}
For binary classification tasks, we get an especially clean and refined mathematical picture, in which boostability is fully determined by a single threshold, i.e., $\val(\cost)$. That is, any learner with loss value below it can be boosted, while any value above it is attainable by a trivial non-boostable learner. In the unweighted case, recent work by \cite{Brukhim23simple} demonstrated that a {\it multichotomy} emerges, governed by  
$k-1$ thresholds: $\frac{1}{2}, \frac{2}{3}, \dots, \frac{k-1}{k}$, where $k$ is the number of labels. For each such threshold, any learner achieving a loss below it can be ``partially'' boosted to the next threshold below it, but not further than that. 
Here, we extend their results to arbitrary cost functions and prove that a similar phenomenon holds more generally.

In contrast to the unweighted case, the landscape of boostability in the general setting is more complex, consisting of many thresholds corresponding to distinct possible subsets of labels. Interestingly, the thresholds of boostability can be computed precisely, and are all determined by the outcome of zero-sum games, as defined next. We generalize the zero-sum game defined in Equation \eqref{eq:zero_sum_game_def} to the case of $k \ge 2$ labels. It turns out that in the multiclass setting, a more refined zero-sum game is needed, introducing a certain asymmetry between the sets of allowable strategies for each player.
In particular, we restrict the set of distributions $\bq$ which the maximizing player is allowed to choose. For a subset $J \subseteq \Y$ of labels define the value of the game \emph{restricted to the subset $J$} as:
\begin{align} \label{eq:def_J_game}
    \val_J(\cost) \triangleq \min_{\bp \in \Delta_\Y} \max_{\bq \in \Delta_J} \cost(\bp,\bq) \;.
\end{align}
Importantly, only the maximizing player is restricted to $\Delta_J$ while the minimizing player is not. Thus, $\val_J(\cost)$ is the smallest loss one can ensure by predicting with a die \emph{given that the correct label is in} $J$. Clearly, for every $J' \subseteq J$ it holds that $\val_{J'}(\cost) \le \val_{J}(\cost)$. 
Consider the subsets $J \subseteq \Y$ in nondecreasing order of $\val_{J}(\cost)$.\footnote{For notational easiness we define $\val_{\emptyset}(\cost)=0$.} We then denote each such cost by $v_{s}(\cost)$ for some $s \in \{1,\dots,2^k\}$, so that overall we have,
\begin{equation}\label{eq:multiclass_critical_thresholds}
 0 = v_1(\cost) < v_2(\cost) < \ldots < v_\tau(\cost)  = \val_{\Y}(\cost),
 \end{equation}
where \( \tau \le 2^k\) depends on the cost function $\cost$. The first equality holds since $\val_{\{y\}}(\cost) = 0$ for $y \in \Y$. Moreover, by a straightforward calculation it can be shown that, for the unweighted case in which $\cost$ is simply the 0-1 loss (i.e., $\cost(i, j) = \ind{i \neq j}$), the thresholds $\val_J(\cost)$ specialize to those given by \citet{Brukhim23simple}. Concretely, it is described in the following fact. 
\begin{fact}
Let $\cost$ be the 0-1 loss. Then, for every $\threshIndex \ge 1$ and  $J \in {\Y \choose \threshIndex}$,
it holds that $\val_J(\cost)=1-\frac{1}{\threshIndex}$.
\end{fact}
Thus, for the 0-1 loss we have $\tau = k$. Note that although for the general case given here 
there is no closed-form expression for the value of each threshold, it can each be determined by solving for the corresponding linear program representation of the zero-sum game described above. \\

\noindent We can now state our main results on multiclass cost-sensitive boosting, given next.

\begin{atheorem}[Cost-sensitive boosting, multiclass case]\label{thm:main_multiclass}
Let $\Y=[k]$, let $\cost: \Y^2 \rightarrow [0,1]$ be any cost function, and let $z \ge 0$. Let $v_{1} < v_{2} < \cdots < v_{\tau}$  as defined in \Cref{eq:multiclass_critical_thresholds},\footnote{When clear from context we omit $\cost$ and denote a threshold by $v_n$.} and let $\threshIndex$ be the largest integer such that $v_{\threshIndex} \le z$. Then, the following claims both hold.
\begin{itemize}[leftmargin=.3cm]\itemsep0pt
    \item{\bf\small $(\cost, z)$ is $\threshIndex$-boostable:} for every $\F \subseteq \Y^\X$, every $(\cost,z)$-learner is boostable to $(\cost,v_{{\threshIndex}})$-learner.
    \item{\bf\small $(\cost, z)$ is not $(\threshIndex-1)$-boostable:} there exists a class $\F  \subseteq \Y^\X$ and a $(\cost,z)$-learner for $\F$ that cannot be boosted to $(\cost,z')$-learner for $\F$, for any $z' < v_{\threshIndex}$.
    \end{itemize} 
\end{atheorem}
The proof of \Cref{thm:main_multiclass} is given in \Cref{sec:multiclass}.
In words, \Cref{thm:main_multiclass} implies that there is a partition of the interval $[0,1]$, on which the loss values lie,  into at most $\tau$ sub-intervals or ``buckets'', based on $v_{{\threshIndex}}$. Then, any loss value $z$ in a certain bucket $\bigl[v_{{\threshIndex}},v_{{\threshIndex+1}}\bigr)$ can be boosted to the lowest value within the bucket, but not below that.
We remark that, in fact, the above boostability result is obtained constructively via an efficient algorithm, as described in \Cref{sec:multiclass}.   \\

\subsubsection{Multiclass classification, multi-objective loss}
We shift our focus to {\it multi-objective} losses in the multiclass setting. 
Recall the notion of $(\costVec, \bbz)$-learner of Definition \ref{def:vec_z_learner}. Our goal is to prove a multi-objective counterpart of \Cref{thm:main_multiclass}; that is, to understand for which $\bbz'$ a $(\costVec, \bbz)$-learner can be  boosted to a $(\costVec, \bbz')$-learner. Unlike the scalar cost setting of \Cref{thm:main_multiclass}, however, the set of those $\bbz'$ that are reachable by boosting a $(\costVec, \bbz)$-learner is not a one-dimensional interval in $[0,1]$, but a subset of $[0,1]^r$. 

As a first step, we generalize the notion of coin attainability given by \Cref{def:coin_attainability}. The key ingredient that we shall add is to restrict the distribution of the labels over a given set $J \subseteq \Y$. From the point of view of a random guesser, this amounts to knowing that the correct label must be in $J$. We then say that a guarantee is $J$-dice-attainable if a trivial learner can satisfy that guarantee over any distribution supported only over $J$. 
Formally, it is defined as follows.
\begin{definition}[{$J$-dice attainability}]\label{def:dice_attainable_J}
    Let $\Y=[k]$, let $\costVec = (\cost_1,\ldots,\cost_r)$ where each $\cost_i: \Y^2 \rightarrow [0,1]$ is a cost function, let $\bz \in [0,1]^{r}$, and let $J \subseteq \Y$.
    Then $(\costVec, \bbz)$ is $J$-\emph{dice-attainable} if:
    \begin{equation}
    \label{eq:J-dice-attainable}
        \forall \bq \in \Delta_J, \ \exists \bp \in \Delta_\Y, \ \forall i = 1, \ldots, r, \qquad \cost_i(\bp,\bq) \le z_i \;.
    \end{equation}
    The \emph{$J$-dice-attainable region} of $\costVec$ is the set $D_J(\costVec) \triangleq \left\{ \bbz  : (\costVec, \bbz) \text{ is }J\text{-dice-attainable}\right\}$.
\end{definition}
\noindent Intuitively, if $(\costVec,\bbz)$ is $J$-dice-attainable, then a $(\costVec,\bbz)$-learner is not better than a random die over $J$, and therefore should not be able to separate any label in $J$.
Using the notion of $J$-dice-attainability, we define a partial ordering over $[0,1]^r$. For every $\bbz, \bbz'
\in [0,1]^r$ write $\bbz \preceq_{\costVec} \bbz'$ if
\begin{equation}\label{eq:partial_order_MC_MO}
    \forall J \subseteq \Y, \ \ \bbz \in D_J(\costVec) \implies \bbz' \in D_J(\costVec)\,,
\end{equation}
and $\bbz \not\preceq_{\costVec} \bbz'$ otherwise. Intuitively, $\bbz \preceq_{\costVec} \bbz'$ means that, whenever $\bbz$ is not better than a die, then $\bbz'$ is not better than a die either. 
We prove that this partial ordering precisely characterizes the boostability of $\bbz$ to $\bbz'$. We can now state our main result for multi-objective multiclass boosting.
\begin{atheorem}[Multi-objective boosting, multiclass case]\label{thm:multiclass_MO_boost}
Let $\Y=[k]$, let $\costVec = (\cost_1,\ldots,\cost_r)$ where each $\cost_i: \Y^2 \rightarrow [0,1]$ is a cost function, and let $\bbz \in [0,1]^r$.  
Then, the following claims both hold. 
\begin{itemize}[leftmargin=.6cm]\itemsep0pt
    \item     {\bf\small $(\costVec, \bbz)$ is boostable to $(\costVec, \bbz')$ for every $\bbz \preceq_{\costVec} \bbz'$}: for every class $\F \subseteq \Y^\X$, every $(\costVec,\bbz)$-learner for $\F$ is boostable to a $(\costVec,\bbz')$-learner for $\F$. 
    \item  {\bf\small $(\costVec, \bbz)$ is not boostable to $(\costVec, \bbz')$ for every $\bbz \not\preceq_{\costVec} \bbz'$}:
    there exists a class $\F \subseteq \Y^\X$ and a $(\costVec, \bbz)$-learner for $\F$ that is a trivial learner and, therefore, cannot be boosted to a $(\costVec, \bbz')$-learner for $\F$, for any $\bbz \not\preceq_{\costVec} \bbz'$. 
\end{itemize}
\end{atheorem}

\noindent Let us elaborate briefly on \Cref{thm:multiclass_MO_boost}.
In the multiclass cost-sensitive setting, where the overall cost is a scalar, the landscape of boostability is determined by a sequence of (totally ordered) scalar thresholds $v_{{\threshIndex}} \in [0,1]$; and those thresholds determine whether a $(\cost, z)$-learner can be boosted to a $(\cost, z')$-learner, for $z' < z$, as detailed in \Cref{thm:main_multiclass}.
In the multiclass multi-objective setting, those scalar thresholds are replaced by a set of surfaces in $[0,1]^r$.  Each such surface corresponds to the boundary of the dice-attainable sets $D_J(\costVec)$. Those surface can be seen as thresholds of a higher-dimensional form, separating between different boostability guarantees.

%% file: 3-binary.tex
\section{Binary Classification}\label{sec:binary}
This section considers the classic binary classification setting, where $\Y=\{-1,+1\}$. 
\Cref{subsec:binary:scalar} tackles the simplest case, where the cost $\cost$ is scalar, and explores the conditions that make a $(\cost,z)$-learner boostable, proving \Cref{thm:intro_binary_boost}. \Cref{subsec:binary:vector} considers multi-objective costs, where the cost $\costVec$ is vector-valued, for a specific but illustrative choice of $\costVec$, proving a special case of \Cref{thm:binary_MO_boost} and of \Cref{prop:intro_coin_duality}. The reason for restricting $\costVec$ to a special form is that this allows us to introduce the key ideas while avoiding heavy notation.
Finally, in \Cref{subsec:binary:pf_MO} we give the full proof of \Cref{thm:binary_MO_boost}.\\

For convenience, as $\Y=\{-1,+1\}$, we will sometimes encode a cost function $\cost : \Y^2 \to [0,1]$ as a 2-by-2 matrix $\bigl(\begin{smallmatrix}0 & w_{-}\\w_{+} & 0\end{smallmatrix}\bigr)$, where $w_{+} = w(+1,-1)$ is the cost of a false positive and $w_{-} = w(-1,+1)$ the cost of a false negative. We will also use the fact, provable through easy calculations, that the value of the game $\val(\cost)$ equals $0$ if $\min(w_-,w_+)=0$ and $\frac{w_- w_+}{w_- + w_+}$ otherwise.
Before moving forward, we also need a measure of the advantage that a $(\cost,z)$-learner has compared to a coin.
\begin{definition}[Margin, binary case]\label{def:binary-margin}
The \emph{margin} of a $(\cost,z)$-learner is $\gamma \triangleq \max\{0, \val(\cost) - z\}$.
\end{definition}
\noindent 
When $w_+ = w_- = 1$, and thus $\cost$ is the standard 0-1 cost, $\gamma$ is the usual notion of margin that defines a weak learner in binary classification tasks. In that case, it is well known that to boost a weak learner with margin $\gamma$ it is necessary and sufficient to invoke the learner a number of times that scales with $\nicefrac{1}{\gamma^2}$. This remains true for general $\cost$ and $z$, in that our boosting algorithm below invokes the $(\cost,z)$-learner a number of times proportional to $\nicefrac{1}{\gamma^2}$.

\subsection{Boosting a $(\cost,z)$-learner}\label{subsec:binary:scalar}
The main contribution of this subsection is a boosting algorithm, shown in \Cref{alg:binary_boost}, that turns any $(\cost,z)$-learner $\A$ with $z < \val(\cost)$ into a $(\cost,0)$-learner $\A'$.
We prove:
\begin{theorem}\label{thm:binary_scalar_boosting_readable}
Let $\Y=\{-1,+1\}$ and $\F\subseteq \Y^\X$. If $\A$ is a $(\cost,z)$-learner for $\F$ with margin $\gamma>0$, then \Cref{alg:binary_boost} with the choice of parameters of \Cref{thm:binary_scalar_boosting} is a $(\cost,0)$-learner for $\F$.
Under the same assumptions, \Cref{alg:binary_boost} makes $T = O\bigl(\frac{\ln(m)}{\gamma^2}\bigr)$ oracle calls to $\A$, where $m$ is the size of the input sample, and has sample complexity
$m(\eps,\delta) = \widetilde{O}\Bigl(m_0\bigl(\frac{\gamma}{3},\frac{\delta}{2T}\bigr) \cdot \frac{\ln(1/\delta)}{\eps \cdot \gamma^2}\Bigr)$, where $m_0$ is the sample complexity of $\A$ .
\end{theorem}

It is immediate to see that \Cref{thm:binary_scalar_boosting_readable} implies the first item of \Cref{thm:intro_binary_boost}; for the second item see \Cref{thm:bin_scalar_LB} below. In the rest of this subsection we prove \Cref{thm:binary_scalar_boosting_readable}. We do so in two steps: first we prove that \Cref{alg:binary_boost} returns a hypothesis consistent with the input sample (\Cref{thm:binary_scalar_boosting}), and then we show generalization by a compression argument (\Cref{lem:bin_scalar_boost_generalization}).

\begin{algorithm}[t]
\caption{Boosting a binary $(\cost,z)$-learner} \label{algorithm:binary:scalar}
\label{alg:binary_boost}
\begin{algorithmic}[1]
\REQUIRE sample $S = (x_i,y_i)_{i=1}^m$; $(\cost, z)$-learner $\A$; parameters $T, \eta, \widehat m$
\vskip3pt
\STATE Initialize: $D_1(i) = 1$  for all $i=1,...,m$.
\FOR{$t = 1, \ldots, T$}
\STATE Compute the distribution $\D_t \triangleq \frac{D_t}{\sum_{i}D_t(i)}$ over $[m]$.
\STATE Draw a set $S_t$ of $\widehat m$ labeled examples i.i.d.\ from $\D_t$ and obtain $h_t = \A(S_t)$.
\STATE For every $i=1,\ldots,m$ let:
\[
D_{t+1}(i) \triangleq D_t(i) \cdot e^{\eta \cdot \cost(h_t(x_i),y_i)} \;.
\]
\ENDFOR
\STATE Let $F: \X\times\Y \to [0,1]$ such that $F(x,y) \triangleq \frac{1}{T}\sum_{t=1}^T \ind{h_t(x) = y}$ for all $x \in \X, y \in \Y$
\RETURN $\widehat h_S: \X \to \Y$ such that, for all $x \in \X$, 
\[
\widehat h_S(x) \triangleq
\begin{cases}
+1  & \text{if } w_- \cdot F(x,-1) < w_+\cdot F(x,+1) \;, \\
-1  & \text{otherwise.}
\end{cases}
\]
\end{algorithmic}
\end{algorithm}

At a high level, \Cref{alg:binary_boost} follows standard boosting procedures. Given a $(\cost,z)$-learner $\A$ and a labeled sample $S$, \Cref{alg:binary_boost} uses $\A$ to obtain a sequence of hypotheses $h_1,\ldots,h_T$ whose average has loss close to $z$ on each single example in $S$.  This can be achieved through any regret-minimizing algorithm, but for the sake of concreteness our algorithm uses a modified version of Hedge \citep{Freund97decision} where the update step re-weights examples as a function of $\cost$. Unlike standard approaches, the final prediction on a point $x$ is not just the majority vote of the $h_t$'s. Instead, the prediction is constructed by comparing the fraction of $h_t$'s returning $-1$, weighted by $w_-$, and the fraction of $h_t$'s returning $+1$, weighted by $w_+$. The smallest weighted fraction wins, and the corresponding label is returned. We show that this ensures correct labeling of the entire sample with the desired probability.
Formally, we have:
\begin{lemma}\label{thm:binary_scalar_boosting}
Let $\Y=\{-1,+1\}$ and let $\cost:\Y^2\to [0,1]$ be a cost function. Let $\F \subseteq \Y^\X$, and let $\A$ be a $(\cost, z)$-learner for $\F$ with margin $\gamma > 0$ and sample complexity $m_0$.
Fix any $f \in \F$, let $\D$ be any distribution over $\X$, and let $S=\{(x_1,y_1),\ldots,(x_m,y_m)\}$ be a multiset of $m$ examples given by i.i.d.\ points $x_1,\dots,x_m \sim \D$ labeled by $f$.
Finally, fix any $\delta \in (0,1)$. If \Cref{alg:binary_boost} is given $S$, oracle access to $\A$, and parameters $T = \ceil*{\frac{18\ln(m)}{\gamma^2}}$, $\eta = \sqrt{\frac{2\ln(m)}{T}}$, and $\widehat m = m_0\bigl(\frac{\gamma}{3}, \frac{\delta}{T}\bigr)$, then \Cref{alg:binary_boost} makes $T$ calls to $\A$ and returns $\widehat h_S:\X\to\Y$ such that with probability at least $1-\delta$:
\[
     \widehat h_S(x_i) = y_i \qquad \forall i =1,\ldots,m \;.
\]
\end{lemma}
\begin{proof}
Fix any $i \in [m]$. Given that $\max_{t,i} \cost(h_t(x_i),y_i) \le 1$ and that $\eta \le 1$, the standard analysis of Hedge \citep{Freund97decision} shows that:
\begin{align}\label{eq:proof:lemma:1}
    \sum_{t=1}^T \cost(h_t(x_i),y_i) &\le \frac{\ln(m)}{\eta} + \frac{\eta}{2} T +
    \sum_{t=1}^T \E_{j \sim \D_t}\Bigl[\cost(h_t(x_j),y_j) \Bigr] \;.
\end{align}
Dividing both sides by $T$, and using the definitions of $\eta$ and $T$ at the right-hand side, yields:
\begin{align}\label{eq:proof:lemma:2}
    \frac{1}{T}\sum_{t=1}^T \cost(h_t(x_i),y_i) &\le \frac{\gamma}{3} +
    \frac{1}{T}\sum_{t=1}^T \E_{j \sim \D_t}\Bigl[\cost(h_t(x_j),y_j) \Bigr] \;.
\end{align}
For every $t=1,\ldots,T$, since $h_t=\A(S_t)$, since $\A$ is a $(\cost,z)$-learner for $\F$, and by the choice of $\widehat m$:
\begin{align}
    \Pr\left(\E_{j \sim \D_t}\Bigl[\cost(h_t(x_j),y_j) \Bigr] > z + \frac{\gamma}{3}\right) \le \frac{\delta}{T} \;.
\end{align}
By averaging over $T$ and taking a union bound over $t=1,\ldots,T$, we have that, with probability at least $1-\delta$,
\begin{align}
    \frac{1}{T} \sum_{t=1}^T \E_{j \sim \D_t}\Bigl[\cost(h_t(x_j),y_j) \Bigr] \le z + \frac{\gamma}{3} \;.\label{eq:binary_avg_loss_bound}
\end{align}

Now suppose \Cref{eq:binary_avg_loss_bound} holds. We shall prove that $\widehat h_S(x_i) = y_i$ for every $i \in [m]$. First, \Cref{eq:proof:lemma:2} together with the definition of $\gamma$ implies that for every $i \in [m]$:
\begin{align} 
    \frac{1}{T}\sum_{t=1}^T \cost(h_t(x_i),y_i) &\le \frac{2}{3}\gamma + z \le \val(\cost) - \frac{\gamma}{3} < \val(\cost) \;.\label{eq:binary_loss_gamma}
\end{align}
Now recall the definition of $F: \X \times \Y \to [0,1]$ from \Cref{alg:binary_boost}:
\begin{align}
    F(x,y) = \frac{1}{T}\sum_{t=1}^T \ind{h_t(x) = y} \qquad  \forall x \in \X, y \in \Y \;.
\end{align}
Clearly, for any $x \in \X$, we have that $F(x,-1)+F(x,+1)=1$.
Now consider any $i \in [m]$. If $y_i=+1$ then by definition of $\cost$ and $F$:
\begin{align}
    \frac{1}{T}\sum_{t=1}^T \cost(h_t(x_i),y_i) = w_{-}\cdot \frac{1}{T}\sum_{t=1}^T \ind{h_t(x_i) = -1} = w_{-}\cdot F(x_i, -1) \;.   
\end{align}
By \Cref{eq:binary_loss_gamma}, then, $\cost_- \cdot F(x_i,-1) < \val(\cost)$.
By a symmetric argument $w_+\cdot F(x_i,+1) < \val(\cost)$ if $y_i=-1$. It remains to show that for every $x \in \X$ exactly one among $w_- \cdot F(x,-1) < \val(\cost)$ and $w_+ \cdot F(x,+1) < \val(\cost)$ holds. As we have just shown, at least one of the inequalities holds, depending on the value of $y_i$. Now assume towards contradiction that both hold. Recall that $\val(\cost)=\frac{w_- w_+}{w_- + w_+}$. Note that $\gamma > 0$ implies $\val(\cost)>0$ and thus $w_-,w_+ > 0$. Then we have:
\begin{align}
1 = F(x,-1) +  F(x,+1) < \frac{\val(\cost)}{w_-} + \frac{\val(\cost)}{w_+} = \frac{w_+}{w_- + w_+} +  \frac{w_-}{w_- + w_+} = 1 \;.
\end{align}
By definition of $\widehat h_S$ then $\widehat h_S(x_i)=y_i$ for all $i \in [m]$, concluding the proof.
\end{proof}

The next lemma shows that, if the size $m$ of the sample in \Cref{thm:binary_scalar_boosting} is large enough, then the hypothesis returned by \Cref{alg:binary_boost} has small generalization error. It is immediate to see that, together with \Cref{thm:binary_scalar_boosting}, this implies \Cref{thm:binary_scalar_boosting_readable}.
\begin{lemma}\label{lem:bin_scalar_boost_generalization}
Assume the hypotheses of \Cref{thm:binary_scalar_boosting}.
For any $\eps,\delta \in (0,1)$, if the size $m$ of the sample given to \Cref{alg:binary_boost} satisfies
\[
    m \ge \frac{\ln(2/\delta)}{\eps} + m_0\left(\frac{\gamma}{3},\frac{\delta}{2T}\right) \cdot T \cdot \left(1+\frac{\ln(m)}{\eps}\right) \;,
\]
then with probability at least $1-\delta$ the output $\widehat h_S$ of \Cref{alg:binary_boost} satisfies $L_\D^\cost(\widehat h_S) \le \epsilon$. Therefore, \Cref{alg:binary_boost} is a $(\cost,0)$-learner for $\F$ with sample complexity $m(\eps,\delta) = \widetilde{O}\Bigl(m_0\bigl(\frac{\gamma}{3}, \frac{\delta}{2T}\bigr) \cdot \frac{\ln(1/\delta)}{\eps\cdot\gamma^2}\Bigr)$.
\end{lemma}
\begin{proof}
First, we apply \Cref{thm:binary_scalar_boosting} with $\delta/2$ in place of $\delta$; we obtain that, with probability at least $1-\delta/2$, the hypothesis $\widehat h_S$ is consistent with the labeled sample $S$.
Next, we apply a standard compression-based generalization argument (see e.g., Theorem~2.8 from \citet{schapire2012boosting}).  To this end, note that one can construct a compression scheme for $\widehat h_S$ of size $\kappa$ equal to the total size of the samples on which $\A$ is invoked, that is, $\kappa = \widehat m \cdot T$.
By Theorem~2.8 from \citet{schapire2012boosting} with $\delta/2$ in place of $\delta$ we get that, with probability at least $1-\delta/2$,
\begin{align}
L_\D^\cost(\widehat h_S) \le \frac{\kappa \ln(m) + \ln(2/\delta)}{m-\kappa} \;.
\end{align}
Straightforward calculations show that the right-hand side is at most $\eps$ whenever:
\begin{align}
    m &\ge \frac{\ln(2/\delta)}{\eps} + \kappa \cdot \left(1+\frac{\ln(m)}{\eps}\right) \label{eq:m_ge_m0_0}
    \\ &= \frac{\ln(2/\delta)}{\eps} + m_0\left(\frac{\gamma}{3},\frac{\delta}{2T}\right) \cdot T \cdot \left(1+\frac{\ln(m)}{\eps}\right) \;.\label{eq:m_ge_m0}
\end{align}
A union bound completes the first part of the claim, and shows that \Cref{alg:binary_boost} is a $(\cost,0)$-learner for $\F$.
The condition on the sample complexity of \Cref{alg:binary_boost} then follows immediately by definition of $T = O\bigl(\frac{\ln(m)}{\gamma^2}\bigr)$.
\end{proof}

\paragraph{Remark on adaptive boosting.}
Traditional boosting algorithms, such as the well-known AdaBoost \citep{schapire2012boosting}, do not assume prior knowledge of the margin parameter $\gamma$ and adapt to it dynamically during execution. In contrast, the boosting algorithm in Algorithm \ref{alg:binary_boost}, along with our broader approach, requires an initial estimate of $\gamma$ as input. If this estimate is too large, the algorithm may fail. However, this issue can be addressed through a straightforward binary search procedure: we iteratively adjust our guess, halving it when necessary based on observed outcomes. This adds only a logarithmic overhead of $O(\ln(1/\gamma))$ to the runtime, without affecting sample complexity bounds.

We conclude this subsection with a proof of the second part of \Cref{thm:intro_binary_boost}.
\begin{lemma}\label{thm:bin_scalar_LB}
Let $\X$ be any domain, let $\Y=\{-1,+1\}$, and let $\cost = (w_+, w_-) \in (0,1]^2$ be a cost over $\Y$. If $z \ge \val(\cost)$, then there is an algorithm that is a $(\cost,z)$-learner for every $\F \subseteq \Y^\X$. Moreover, for some domain $\X$ and some $\F \subseteq \Y^\X$ the said learner cannot be boosted to a $(\cost,z')$-learner for any $z' < \val(\cost)$.
\end{lemma}
\begin{proof}
By definition of $\val(\cost)$ there exists a distribution $\bp \in \Delta_\Y$ such that $\cost(\bp,\bq) \le \val(\cost)$ for every $\bq \in \Delta_\Y$ over $\Y$.
Consider then the algorithm that ignores the input and returns a randomized predictor $h_{\bp}$ such that $h_{\bp}(x) \sim \bp$ independently for every given $\x \in \X$. Clearly, for every distribution $\D$ over $\X$,
\begin{align}
    L_\D^{\cost}(h_{\bp}) \le \cost(\bp,\bq) \le \val(\cost) \;,
\end{align}
where $\bq$ is the marginal of $\D$ over $\Y$. This proves that the algorithm is a $(\cost,z)$-learner for every $\F \subseteq \Y^\X$. 
Let us now show that for some $\F \subseteq \Y^\X$ cannot be boosted to a $(\cost,z')$-learner for any $z' < \val(\cost)$. If $\val(\cost)=0$ then this is trivial, as $\cost \ge 0$; thus we can assume $\val(\cost)>0$.
Suppose then that a $(\cost,z')$-learner for $\F$ exists, where $z' < \val(\cost)$. By item (1), it follows that a $(\cost,0)$-learner for $\F$ exists.
Recall that $\val(\cost) = \frac{w_- w_+}{w_- + w_+}$. Since we assumed $\val(\cost)>0$, it follows that $w_-,w_+ > 0$. Then, a $(\cost,0)$-learner is a strong learner in the PAC sense.
Now choose $\F$ to be any non-learnable class; for instance, let $\F=\Y^\NN$ (see \citet[Theorem 5.1]{shalev2014understanding}). Then, no strong PAC learner and therefore no $(\cost,0)$-learner exists for $\F$. Hence a $(\cost,z')$-learner for $\F$ with $z' < \val(\cost)$ does not exist, too.
\end{proof}

\subsection{Multi-objective losses: boosting a $(\costVec,\bz)$-learner}\label{subsec:binary:vector}
This section considers the boosting problem for learners that satisfy multi-objective costs, as captured by the notion of $(\costVec,\bz)$-learners (\Cref{def:vec_z_learner}).
Our main question is then:
\begin{center}
    \it When can a $(\costVec,\bz)$-learner be boosted to a $(\costVec,\bzero)$-learner?
\end{center}
For the scalar case, we showed in \Cref{subsec:binary:scalar} that the threshold to boost $(\cost,z)$-learners is the value of the game $\val(\cost)$, which is the best bound on $\cost$ one can achieve by just tossing a coin---independently of the example $x$ at hand, and even of the specific distribution $\D$ over $\X$. One may thus guess that the same characterization holds for boosting $(\costVec,\bz)$-learners.  It turns out that this guess fails, as we show below. We thus adopt a different and stronger notion of ``trivial learner'', obtained by exchanging the order of the players. It turns out that this is the right notion, in the sense that it is equivalent to being non-boostable.

For the sake of simplicity we assume a specific $\costVec$ that we call \emph{population-driven} cost. This will help conveying the key ideas, without having to deal with the technical details of the more general version of our results (\Cref{sec:multiclass}). The population-driven cost is the one that counts, separately, false negatives and false positives. More precisely, $\costVec = (\cost_-, \cost_+)$ where $\cost_-$ and $\cost_+$ are now two cost functions that for all $i,j \in \Y$ satisfy:
\begin{align}
    \cost_-(i,j) &= \ind{i=-1 \land j=+1} \;,\\
    \cost_+(i,j) &= \ind{i=+1 \land j=-1} \;.
\end{align}
For $\bz=(z_-,z_+) \in \RRp^2$, a $(\costVec,\bz)$-learner then ensures simultaneously a false-negative rate arbitrarily close to $z_-$ and a false-positive rate arbitrarily close to $z_+$.

\subsubsection{Coin attainability}
Let $\costVec$ as above, and let $\bz \in \RRp^2$. As a first attempt at answering our question above, we guess that $\bz$ is non-boostable when it is attainable by a random coin toss in the same way as for scalar costs (\Cref{subsec:binary:scalar}). That is, we may say that $\bz$ is \emph{trivially attainable} w.r.t.\ $\costVec$ if there exists $\bp \in \Delta_\Y$ such that for every $\bq \in \Delta_\Y$ we have $\costVec(\bp,\bq) \le \bz$. Clearly, for any such $\bz$ there would exist a $(\costVec,\bz)$-learner that is not boostable for some class $\F$. If we could also prove the converse, that any $\bz$ \emph{not}
trivially attainable is boostable, then we would have determined that trivial attainability is equivalent non-boostability.

Unfortunately, this turns out to be false. To see why, fix any distribution $\bp=(p_-,p_+)$. Define then $\bq=(q_-,q_+)$ as follows: if $p_- \ge \nicefrac{1}{2}$ then $q_-=0$, otherwise $q_+=0$. It follows immediately that either $p_-(1-q_-) \ge \nicefrac{1}{2}$ or $p_+(1-q_+) \ge \nicefrac{1}{2}$. That is, $\costVec(\bp,\bq)$ is at least $\nicefrac{1}{2}$ in at least one coordinate. As a consequence, no $\bz < (\nicefrac{1}{2},\nicefrac{1}{2})$ is trivially attainable. Thus, if our guess is correct, it should be the case that any $(\costVec,\bz)$-learner with $\bz < (\nicefrac{1}{2},\nicefrac{1}{2})$ is boostable.

We can easily show that that is not the case. Let $\bz=(\nicefrac{1}{4},\nicefrac{1}{4})$ and consider the following $(\costVec,\bz)$-learner. Upon receiving a sufficiently large labeled sample from some distribution $\D$ over $\X$, the learner computes an estimate $\hat \bq=(\hat q_-,\hat q_+)$ of the marginal of $\D$ over $\Y$. The learner then computes $\bp=(p_-,p_+)$ that satisfies:
\begin{align}
    p_- \cdot(1-\hat q_-) \le \frac{1}{4} \quad \text{and} \quad p_+ \cdot(1-\hat q_+) \le \frac{1}{4} \;.
\end{align}
It is not hard to show that such a $\bp$ always exists. The learner thus outputs $h$ such that $h(x) \sim \bp$ for all $x \in \X$ independently. As the number of samples received by the learner increases, the estimate $\hat\bq$ approaches $\bq$, and thus the loss of the learner approaches $\bz$. Clearly, however, such a learner cannot be boosted, as its predictions are independent of the example $x$ at hand. We conclude that there exist $\bz$ that are not trivially attainable and yet are not boostable.

The example above suggest a fix to our notion of ``non-boostable learner''. The fix consists in considering again prediction by a coin $\bp$, where, however, we allow $\bp$ to be a function of $\bq$. This leads to the following definition:
\begin{definition}\label{def:bin_vec_coin}
    Let $\costVec = (w_-,w_+)$ be the population-driven cost and let $\bz = (z_-,z_+) \in \RRp^2$.
    Then, $\bz$ is \emph{coin-attainable} w.r.t.\ $\costVec$ if for every distribution $\bq=(q_-,q_+)$ there exists a distribution $\bp=(p_-,p_+)$ such that:
    \begin{equation*}
        \cost_-(\bp,\bq) = p_- \cdot(1-q_-) \le z_- \quad \text{and} \quad \cost_+(\bp,\bq) = p_+ \cdot(1-q_+) \le z_+ \;.
    \end{equation*}
\end{definition}
\noindent Note how \Cref{def:bin_vec_coin} mirrors the game of \Cref{subsec:binary:scalar}, but with a reverted \emph{order of the players}: here, the predictor plays second.
It should be clear that, if $\bz$ is coin-attainable, then there exists a $(\costVec,\bz)$-learner that is not boostable. That is the learner described above, which gets an estimate $\hat\bq$ of $\bq$ from the sample, and then returns $h$ such that $h(x) \sim \bp$, where $\bp$ satisfies the condition in \Cref{def:bin_vec_coin} with respect to $\hat\bq$. What is not clear is whether the converse holds, too: is a $(\costVec,\bz)$-learner boostable whenever $\bz$ is \emph{not} coin-attainable? It turns out that this is the case, as the next subsection shows.

\subsubsection{Coin-attainability equals non-boostability}
Let us try to boost a $(\costVec,\bz)$-learner $\A$. We can do so by reduction to the scalar case of \Cref{subsec:binary:scalar}.
To this end choose any $\balpha = (\alpha_1,\alpha_2) \in \Delta_\Y$. Define $\balpha \cdot \costVec : \Y^2 \to \RRp$ by:
\begin{align}
    (\balpha \cdot \costVec)(i,j) = \alpha_1 \cdot \cost_-(i,j) + \alpha_2 \cdot \cost_+(i,j) \qquad \forall i,j \in \Y \;.
\end{align}
For conciseness we shall write $\cost = \balpha \cdot \costVec$. In words, $\cost$ is the convex combination of $\cost_-,\cost_+$ according to $\balpha$, and therefore $\cost=\bigl(\begin{smallmatrix}0 & \alpha_1\\\alpha_2 & 0\end{smallmatrix}\bigr)$. Therefore $\cost:\Y^2\to [0,1]$ and $\cost(i,i)=0$ for each $i \in \Y$. Now let $z = \balpha \cdot \bz$. Now we make the following crucial observation: since $\A$ is a $(\costVec,\bz)$-learner, then it is also a $(\cost,z)$-learner. This is immediate to see from the definition of $(\costVec,\bz)$ and the definition of $\cost$. It follows that, by the results of \Cref{subsec:binary:scalar}, $\A$ can be boosted if $z < \val(\cost)$. Summarizing, we can boost $\A$ whenever:
\begin{align}
    \balpha \cdot \bz < \val(\balpha \cdot \costVec) \;.\label{eq:bin_scalarization}
\end{align}
That is, if there is \emph{any} $\balpha \in \Delta_\Y$ satisfying \Cref{eq:bin_scalarization}, then $\A$ can be boosted to a strong learner. Moreover, if that is the case, then one can prove that one can ensure $\balpha > \bzero$. Then, one can easily see that $\A^*$ is a $(\costVec,\bzero)$-learner. Our main question thus becomes:
\begin{center}
    \it Is it true that, if $\bz$ is not coin-attainable, then there is $\balpha \in \Delta_\Y$ such that $\balpha \cdot \bz < \val(\balpha \cdot \costVec)$?
\end{center}
We show that this is indeed the case, and therefore $\bz$ is coin-attainable if and only if $\balpha \cdot \bz \ge \val(\balpha \cdot \costVec)$ for all $\balpha \in \Delta_\Y$. Equivalently, all $(\costVec,\bz)$-learners are boostable if and only if $\balpha \cdot \bz < \val(\balpha \cdot \costVec)$ for some $\balpha \in \Delta_\Y$. As a consequence, every $\bz$ is either coin-attainable or boostable. Formally, we have:
\begin{theorem}\label{thm:bin_sqrtz_duality}
    Let $\Y=\{-1,+1\}$, let $\costVec$ be the population-driven cost, and let $\bz = (z_-,z_+) \in \RRp^2 \setminus \bzero$.
    Then $\bz$ is coin-attainable w.r.t.\ $\costVec$ if and only if $\balpha \cdot \bz \ge \val(\balpha \cdot \costVec)$, where: 
    \begin{align}
        \balpha = \left(\frac{\sqrt{z_+}}{\sqrt{z_-}+\sqrt{z_+}}, \frac{\sqrt{z_-}}{\sqrt{z_-}+\sqrt{z_+}}\right) \in \Delta_\Y\,.
    \end{align}
    Equivalently, $\bz$ is coin-attainable w.r.t.\ $\costVec$ if and only if $\sqrt{z_-}+\sqrt{z_+} \ge 1$.
\end{theorem}
\begin{proof}
    First of all, by straightforward calculations we determine:
    \begin{align}
        \balpha \cdot \bz &= \sqrt{z_-} \sqrt{z_+} \;,\\
        \val\left(\balpha \cdot \costVec\right) &= \val\left(\begin{smallmatrix}0 & \alpha_1\\\alpha_2 & 0\end{smallmatrix}\right) =\frac{\sqrt{z_-}\sqrt{z_+}}{\sqrt{z_-}+\sqrt{z_+}} \;.
    \end{align}
    This proves that $\balpha \cdot \bz \ge \val(\balpha \cdot \costVec)$ if and only if $\sqrt{z_-}+\sqrt{z_+} \ge 1$.
    
    Now, for the ``only if'' direction, by von Neumann's minimax theorem and by definition of $\balpha$,
    \begin{align}
        V(\balpha\cdot\costVec) &=
        \max_{\bq\in\Delta_\Y} \min_{\bp\in\Delta_\Y} \frac{\sqrt{z_+}}{\sqrt{z_-}+\sqrt{z_+}}\cdot p_-\cdot(1-q_-) + \frac{\sqrt{z_-}}{\sqrt{z_-}+\sqrt{z_+}}\cdot p_+\cdot(1-q_+)
        \\
        &\le \frac{\sqrt{z_+} \cdot z_-}{\sqrt{z_-}+\sqrt{z_+}} + \frac{\sqrt{z_-}\cdot  z_+}{\sqrt{z_-}+\sqrt{z_+}}
        = \sqrt{z_-} \sqrt{z_+}
        = \balpha\cdot\bz \;,
    \end{align}
    where the inequality holds as $\bz$ is coin-attainable.

    For the ``if'' direction, suppose $\balpha \cdot \bz \ge \val(\balpha \cdot \costVec)$, hence $\sqrt{z_-} + \sqrt{z_+} \ge 1$. We claim that this implies $\bz$ is coin-attainable. Fix indeed any $\bq = (q_-,q_+) \in \Delta_\Y$. We consider three cases:  (1) $q_+ \le z_-$, (2) $q_- \le z_+$, and (3) $q_+ > z_-$ and $q_- > z_+$.  If $q_+ \le z_-$ then choose $\bp=(1,0)$, otherwise if $q_- \le z_+$ choose $\bp=(0,1)$. It is immediate to see that this implies $p_-\cdot(1-q_-)\le z_-$ and $p_+\cdot(1-q_+) \le z_+$. Suppose then $q_+ > z_-$ and $q_- > z_+$, which implies $\bq > \bzero$. We claim that:
    \begin{align}
        \frac{z_-}{q_+} + \frac{z_+}{q_-} \ge 1 \;.\label{eq:zq_ge1}
    \end{align}
    Indeed, letting $a=z_-$, $b=z_+$, $x=q_+$, multiplying both sides by $x(1-x)$, and rearranging yields:
    \begin{align}
        a \cdot (1-x) + b \cdot x - x(1-x) \ge 0 \;,
    \end{align}
    which can be easily checked to hold for all $x$ whenever $\sqrt{a}+\sqrt{b} \ge 1$. \Cref{eq:zq_ge1} then implies the existence of $p_-,p_+$ such that $p_-+p_+=1$ and that:
    \begin{align}
        p_- \cdot (1-q_-) &\le \frac{z_-}{q_+} \cdot (1-q_-) = z_- \;,
        \\
        p_+ \cdot (1-q_+) &\le \frac{z_+}{q_-} \cdot (1-q_+) = z_+ \;.
    \end{align}
    We conclude that $\bz$ is coin-attainable.
\end{proof}

\subsection{A duality, and a geometric interpretation}
\begin{figure}[h]
    \centering
    \includegraphics[width=200pt]{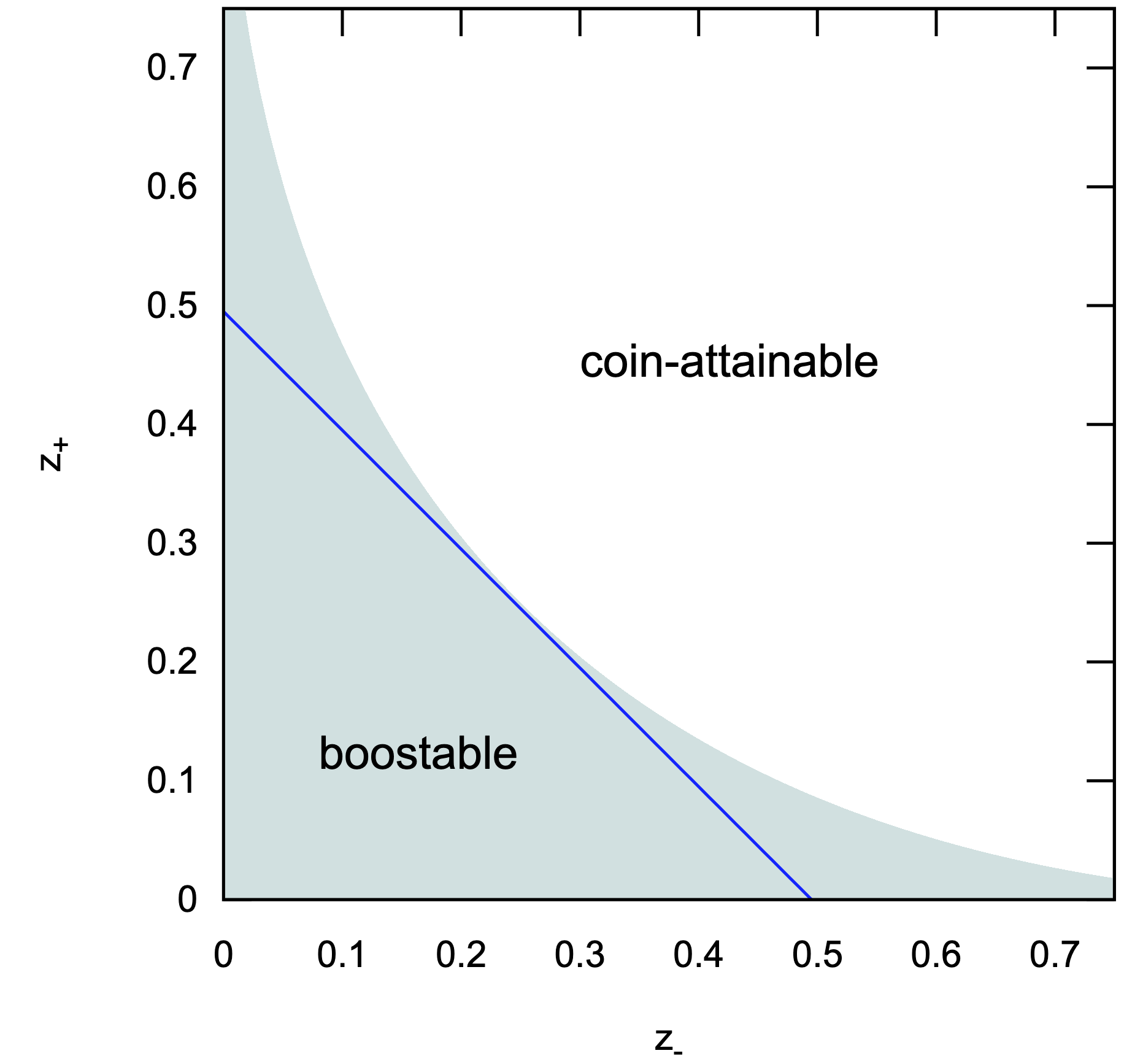}
    \caption{Duality in a picture. For a multi-objective cost $\costVec:\Y^2\to\RRp^2$, the set $\C(\costVec)$ of all coin-attainable vectors $\bz$ is the intersection of all halfspaces in the form $H(\balpha)=\{\bx \in \RR^2 : \balpha \cdot \bx \ge \val(\balpha \cdot \costVec)\}$. As a consequence, the complement of $\C(\costVec)$ is the set of all boostable vectors $\bz$. This is a special case of a more general result, stated in \Cref{sec:multiclass}, that holds for any $k \ge 2$ and any multi-objective cost $\costVec : \Y^2 \to \RRp^r$. The straight blue line is the boundary of a specific $H(\balpha)$: for any $\bz$ in it, every $(\costVec,\bz)$-learner is boostable w.r.t.\ the cost $\balpha \cdot \costVec$. }
    \label{fig:binary-coin_vs_boost}
\end{figure}
\Cref{thm:bin_sqrtz_duality} has an interesting geometric interpretation. In fact, one can easily check that for every $\bz \in \RRp^2$ it holds that $\sqrt{z_-}+\sqrt{z_+} \ge 1$ if and only if $\balpha \cdot \bz \ge \val(\balpha \cdot \costVec)$ for every $\balpha \in \Delta_\Y$, and not just for the single $\balpha$ specified in \Cref{thm:bin_sqrtz_duality}. That is, we have:
\begin{theorem}\label{thm:bin_duality}
    Let $\Y=\{-1,+1\}$, let $\costVec$ be the population-driven cost, and let $\bz \in \RRp^2$. Then $\bz$ is coin-attainable w.r.t.\ $\costVec$ if and only if $\balpha \cdot \bz \ge \val(\balpha \cdot \costVec)$ for every $\balpha \in \Delta_\Y$.
\end{theorem}
\noindent \Cref{thm:bin_duality} says that $\bz$ is coin-attainable if and only if there is no way to ``scalarize'' $\costVec$ so as to obtain a boostable learner w.r.t.\ the corresponding scalarized cost $z$. From a geometric perspective, the result can be read as follows. For any $\balpha \in \Delta_\Y$ consider the following halfspace:
\begin{align}
    H(\balpha) = \{\bx \in \RR^2 \,:\, \balpha \cdot \bx \ge \val(\balpha \cdot \costVec) \} \;. \label{eq:bin:Ha}
\end{align}
Let moreover $\C(\costVec)$ be the set of all $\bz$ that are coin-attainable w.r.t.\ $\costVec$. Then \Cref{thm:bin_duality} says:
\begin{align}
    \C(\costVec) = \bigcap_{\balpha \in \Delta_\Y} H(\balpha) \;.
\end{align}
In other words, the set of coin-attainable vectors is precisely the intersection of all halfspaces $H(\balpha)$. It follows that $\C(\costVec)$ is convex: and, indeed, if two vectors $\bz$ and $\bz'$ are both coin-attainable then it is easy to see that $a \cdot \bz + (1-a)\cdot \bz'$ is coin-attainable, too, for every $a \in [0,1]$.
\Cref{fig:binary-coin_vs_boost} gives a pictorial description of this geometric interpretation. The proof of \Cref{thm:bin_duality} is not given here, as it is a special case, for $\costVec$ being the population-driven cost, of \Cref{prop:intro_coin_duality}.

\subsection{Proof of \Cref{thm:binary_MO_boost}}\label{subsec:binary:pf_MO}
\textit{First item}. The proof makes use of the duality results of \Cref{sec:duality}. Assume $(\costVec, \bbz) \notin \C(\costVec)$; that is, $\bz$ is not coin-attainable with respect to $\costVec$. By \Cref{thm:bin_duality} there exists $\balpha \in \Delta_r$ such that $z_{\balpha} <  \val(\cost_{\balpha})$, where $w_{\balpha} = \sum_{i=1}^r \alpha_i w_i$. Thus, $\A$ is a $(w_{\balpha}, z_{\balpha})$-learner for $\F$. By \Cref{thm:bin_scalar_LB}, then, there exists a $(\cost_{\balpha}, 0)$-learner for $\F$. It is immediate to see that, if $\balpha > \bzero$ (that is, $\alpha_\ell>0$ for every $\ell \in [r]$) then 
it is a $(\cost,0)$-learner for $\F$. Thus, we need only to show that we can always assume $\balpha > \bzero$.
Let
\begin{align}
    \balpha(\eps) = (1-\eps)\balpha + \frac{\eps}{r} \cdot \bone \quad \forall \eps \in [0,1]\,.
\end{align}
Note that both $z_{\balpha(\eps)}$ and $\val(\cost_{\balpha(\eps)})$ are continuous functions of $\eps$. Since $z_{\balpha(0)} = z_{\balpha} < \val(\cost_{\balpha}) = \val(\cost_{\balpha(0)})$, it follows that there exists $\eps > 0$ such that $z_{\balpha(\eps)} < \val(\cost_{\balpha(\eps)})$, too.

\textit{Second item}. We show that, if $\bz \in \C(\costVec)$, then there is a trivial learner $\A$ that is a $(\costVec,\bz)$-learner for every $\F=\Y^\X$. By choosing $\F$ to be any non-learnable class this implies that $\A$ is not boostable to a $(\costVec,\bz')$-learner for $\F$ for any $\bz' \notin \C(\costVec)$, for otherwise the first item above would imply the existence of a $(\costVec,\bzero)$-learner for $\F$.
The learner $\A$ works as follows. Given a sample $S$ of $m$ examples drawn i.i.d.\ from $\D$ and labeled by any $f : \X \to \Y$, it estimates the marginal $\bq$ of $\D$ on $\Y$. Denote that estimate by $\hat \bq$.
It is well known that $\A$ can ensure $\hat \bq$ is within total variation distance $\eps$ of $\bq$ with probability at least $1-\delta$ by taking $m=\Theta\bigl(\frac{k+\ln(1/\delta)}{\eps^2}\bigr)$, see for instance \citep[Theorem 1]{canonne2020shortnotelearningdiscrete}.
Assume then this is the case. Let $\bp \in \Delta$ be a distribution such that $\costVec(\bp,\hat\bq) \le \bz$; note that $\bp$ exists by definition of coin attainability.
Then, $\A$ returns a randomized predictor $h_S : \X \to\Y$ such that $h_S(x) \sim \bp$ independently of the input $x$.
Since $\costVec \in [0,1]^r$, then,
\begin{align}
    L_\D^{\costVec}(h_S) \le \costVec(\bq,\hat \bq) + \eps \cdot \bone = \bz + \eps \cdot \bone \,.
\end{align}
This proves that $\A$ is a $(\costVec,\bz)$-learner.

%% file: 4-duality.tex
\section{Equivalence: Cost-Sensitive and Multi-Objective}\label{sec:duality}

In this section we examine the connection between cost-sensitive and multi-objective learning. In particular, 
we prove the two equivalence theorems: \Cref{thm:intro_general_duality} which demonstrates a certain equivalence between learning algorithms for a class $\F$, with respect to these different guarantees, 
and \Cref{prop:intro_coin_duality}, which characterizes trivial guarantees in both settings.

 The following lemma is the key component used to prove   \Cref{prop:intro_coin_duality}. 

\begin{lemma}\label{lem:proving_duality}
Let $r \in \mathbb{N}$, let $\costVec = (\cost_1,\ldots,\cost_r)$ where each $\cost_i:\Y^2 \to [0,1]$ is a cost function, and let $\bbz \in [0,1]^r$. 
  Fix any $J \subseteq \Y$, and  $\bq \in \Delta_J$.
  Let $\H \subseteq \Y^\X$ be some fixed hypotheses class. 
    Assume that for every $\balpha \in \Delta_r$, there exists $\bp \in \Delta_\H$ such that $\sum_{i=1}^r \alpha_i \cdot w_i(\bp, \bq) \le  \langle \balpha, \bbz\rangle$. Then, there exists $\bp^* \in \Delta_\H$ such that for all $i=1,...,r$, it holds that $w_i(\bp^*, \bq) \le  z_i$. 
\end{lemma}
\begin{proof}
    Define the following zero-sum game. 
The pure strategies of the maximizing player are $[r]$, and the pure strategies minimizing player are $\Y$. The payoff matrix is given as follows, with rows corresponding to labels $\ell \in \Y$ and columns to loss-entry $i \in [r]$,
\begin{equation}
    M(\ell, i) =  \E_{j\sim \bq} \ w_i(\ell,j) - z_i.
\end{equation}
Thus, we have,
\begin{equation}
    M(\bp, \balpha) \triangleq \E_{\ell\sim \bp} \ \E_{i\sim \balpha} \ M(\ell, i)
= \sum_{i=1}^r \alpha_i w_i(\bp,\bq) - \langle\balpha,\bbz\rangle.
\end{equation}
Denote, 
\begin{align} 
    \val(\bq) \triangleq 
    \max_{\balpha \in \Delta_r} \min_{\bp \in \Delta_\Y} M(\bp,\balpha).
\end{align}
It remains to show that there exists $\bp$ such for all $i=1,...,r$, we have $w_i(\bp,\bq) \le z_i$. 
By assumption, we know that $\val(\bq) \le 0$, and so by von Neumann’s minimax theorem \citep{neumann1928theorie} we have:
\begin{align} 
    \min_{\bp \in \Delta_\Y} \max_{\balpha \in \Delta_r} M(\bp,\balpha) \le 0.
\end{align}
In particular, there {exists} $\bp^* \in \Delta_\Y$ such that for all $\balpha \in \Delta_r$ it holds that:
$$
\sum_{i=1}^r \alpha_i w_i(\bp^*,\bq) \le \langle\balpha,\bbz\rangle.
$$
Observe that the above holds for all $\balpha$ of the form $\mathbf{e}_i$ for $i=1,...,r$ (i.e., the standard basis of $\RR^r$), and so we get the desired result. 
\end{proof}

\begin{proposition}\label{prop:J_dice_duality}
Let $r \in \mathbb{N}$, and let $\costVec = (\cost_1,\ldots,\cost_r)$ where each $\cost_i:\Y^2 \to [0,1]$ is a cost function, $\bbz \in [0,1]^r$, and $J \subseteq \Y$. Then, the following are equivalent.
\begin{enumerate}
    \item $(\costVec, \bbz)$ is $J$-dice-attainable. 
    \item $\forall \balpha \in \Delta_r$,  it holds that $\langle\balpha,\bbz\rangle \ge  \val_J(w_{\balpha})$, where $w_{\balpha} = \sum_{i=1}^r \alpha_i w_i$. 
\end{enumerate}
\end{proposition}
Next, we prove  \Cref{prop:J_dice_duality}, which also proves \Cref{prop:intro_coin_duality} as a special case. 
\begin{proof}[Proof of \Cref{prop:J_dice_duality}]
It is easy to see that $1 \implies 2$ holds, by definitions. We prove $2 \implies 1$.
    By assumption,  for every $\balpha \in \Delta_r$  it holds that,
    \begin{equation}
        \val_J(w_{\balpha}) =  \max_{\bq \in {\Delta_J}} \min_{\bp \in \Delta_{\Y}} \sum_{i=1}^r  \alpha_i \cdot w_i(\bp, \bq) \le  \langle \balpha, \bbz\rangle.
    \end{equation}
    In particular, this implies that for any $\bq \in \Delta_J$ and any $\balpha \in \Delta_r$, there exists $\bp \in \Delta_\Y$ such that 
    $$\sum_{i=1}^r \alpha_i \cdot w_i(\bp, \bq) \le  \langle \balpha, \bbz\rangle.$$ Thus, by \Cref{lem:proving_duality} we get that for any $\bq \in \Delta_J$, there exists $\bp^* \in \Delta_\Y$ such that for every $i \in [r]$,
    $$ w_i(\bp^*, \bq) \le  z_i.$$
    Thus, we get that $(\costVec, \bbz)$ is $J$-dice-attainable. 
\end{proof}

We shift our focus to general learning algorithms for a class $\F$. First, recall that a multi-objective learner is, by definition, also a cost-sensitive learner. The existence of a  $(\costVec, \bbz)$-learner implies the existence of a $(w_i,z_i)$-learner for each $i=1,\ldots,r$. Equivalently, the existence of a $(\costVec, \bbz)$-learner implies the existence of a $(w_{\balpha},z_{\balpha})$-learner for $\balpha\in\{\bbe_1,\ldots,\bbe_r\}$ (the canonical basis of $\RR^r)$ where $w_{\balpha} = \sum_{i=1}^r \alpha_i w_i$ and $z_{\balpha} = \langle\balpha,\bbz\rangle$. \\

\Cref{thm:intro_general_duality} shows the other direction holds, i.e., that the existence of a $(w_{\balpha},z_{\balpha})$-learner for all possible $\balpha$ implies the existence of a $(\costVec, \bbz)$-learner.

\begin{algorithm}[h]
\caption{Boosting $(w_{\balpha}, z_{\balpha})$-learners} 
\label{alg:CS_to_MO}
\begin{algorithmic}[1]
\REQUIRE Sample $S = (x_j,y_j)_{j=1}^m$; $(w_{\balpha}, z_{\balpha})$-learners $\A_{\balpha}$ for every $\balpha \in \Delta_r$; parameters $T, \eta, \widehat m$
\vskip3pt
\STATE Initialize: $a_1(i) = 1$  for all $i=1,...,r$, set $U$ to be the uniform distribution over $[m]$.
\STATE Define $M(h, i) =  \ind{\frac{1}{m}\sum_{j=1}^m \ w_i(h(x_j),y_j) > z_i}$ for all $i=1,...,r$ and $h: \X \rightarrow \Y$. 
\FOR{$t = 1, \ldots, T$}
\STATE Compute the distribution $\balpha_t \triangleq \frac{\mathbf{a}_t}{\sum_{i}a_t(i)}$ over $[r]$.
\STATE Draw a set $S_t$ of $\widehat m$ labeled examples i.i.d.\ from $U$ and obtain $h_t = \A_{\balpha_t}(S_t)$.
\STATE For every $i=1,\ldots,r$ let:
\[
a_{t+1}(i) \triangleq a_t(i) \cdot e^{\eta \cdot M(h_t,i)} \;.
\]
\ENDFOR
\RETURN The set of weak predictors $h_1,...,h_T$.
\end{algorithmic}
\end{algorithm}

\begin{proof}[Proof of \Cref{thm:intro_general_duality}]
    It is easy to see that $1 \implies 2$ holds, by definitions. We prove $2 \implies 1$.   
For any $f\in \F$, and distribution $\D$ over $\X$, let $S$ be a set of $m$ examples sampled i.i.d. from $\D$ and labeled by $f$, where $m$ to be determined later. Let $\epsilon,\delta > 0$. 
Then, applying \Cref{alg:CS_to_MO} over $S$ with access to $\A_{\balpha}$ learners, and 
with $T = {100r^2\ln(r)}$, $\eta = \sqrt{\frac{2\ln(r)}{T}}$, $\widehat m = m_0(\frac{1}{10r}, \frac{1}{20r T})$ (where $m_0$ is the sample complexity of $\A_{\balpha}$) yields the following. Fix any $i \in [r]$.  
Given that $\max_{t,i} M(h_t,i) \le 1$ and that $\eta \le 1$, the standard analysis of Hedge \citep{Freund97decision} shows that:
\begin{equation}
    \sum_{t=1}^T M(h_t, i)  \le \frac{\ln(r)}{\eta} + \frac{\eta}{2} T +
    \sum_{t=1}^T \E_{i \sim \balpha_t}\Bigl[M(h_t, i)\Bigr] \;.
\end{equation}
By the guarantee of $\A_{\balpha}$ and the choice of $\widehat m$, and by union bound over all $t \in [T]$ we get that with probability at least $1-\frac{1}{20r}$,  
\begin{equation}
    \frac{1}{T}\sum_{t=1}^T \frac{1}{m}\sum_{j=1}^m \ w_{\balpha_t}(h_t(x_j),y_j) \le \frac{1}{T}\sum_{t=1}^T z_{\balpha_t} + \frac{1}{10r} \;,
\end{equation}
which implies,
\begin{equation}
    \frac{1}{T}\sum_{t=1}^T \E_{i \sim \balpha_t}\Bigl[M(h_t, i) \Bigr] \le   \frac{1}{10r}.
\end{equation}
Thus, we get that with probability $1-\frac{1}{20r}$, 
\begin{equation}
    \frac{1}{T}\sum_{t=1}^T M(h_t, i)  \le \frac{\ln(r)}{T\eta} + \frac{\eta}{2} + \frac{1}{10r} \le \frac{1}{5r} \;.
\end{equation}
By then also taking a union bound over $i=1,...,r$, we have with probability at least $\frac{19}{20}$ for all $i=1,...,r$ it holds that,
\begin{equation} 
   \frac{1}{T}\sum_{t=1}^T M(h_t, i)  \le  \frac{1}{5r}.
\end{equation}
Recall that for each $t$ and $i$ we have that $M(h_t, i) \in \{0,1\}$, and so for all $t\in [T]$ but $\frac{T}{5r}$ it holds that $M(h_t, i)=0$.  We can now define the randomized hypothesis $\Tilde{h}$ for any $x\in\X$ as follows. To compute the value of $\Tilde{h}(x)$,  first we sample a hypothesis $h_t$ by sampling $t \sim U(T)$ uniformly at random, and then predict according to $h_t(x)$. We then have,
\begin{equation} 
     \Pr_{t \sim U(T)}\Big[\sum_{i=1}^r M(h_t, i) \ge 1  \Big]  \le \sum_{i=1}^r \Pr_{t \sim U(T)}\Big[M(h_t, i) = 1  \Big]  \le \frac{1}{5}. 
\end{equation}
Then, by the definition of $M(h,i)$, we have that with probability at least $\frac{19}{20}$,
\begin{equation} 
    \Pr_{\Tilde{h}}\Big[\forall i = 1,...,r,  \ \ \frac{1}{m}\sum_{j=1}^m \ w_i(\Tilde{h}(x_j),y_j) \le z_i \Big] \ge \frac{4}{5}.
\end{equation}
Finally, we apply standard compression-based generalization arguments (\cite{shalev2014understanding}, Theorem 30.2) 
where the sample size is $m = \tilde{O}(\frac{\kappa }{\epsilon^2})$ and where the compression size is $\kappa=\widehat m \cdot T$, we get that with probability at least $1-\delta$ over the sample of $S \sim \D^m$ it holds that:
\begin{equation}\label{eq:boosted_CS_to_MO}
    \Pr_{\tilde{h}}\Big[\forall i = 1,...,r,  \ \ L_{\D}^{w_i}(\tilde{h}) \le z_i +\epsilon \Big] > 1/2.
\end{equation}
Thus, the randomized hypothesis $\Tilde{h}$ satisfies the desired guarantees with a constant probability. 
Overall we get that the procedure described above is a
$(\costVec,\bbz)$-learner for $\F$ in expectation.  \\

Lastly, we remark that instead of a randomized learner whose guarantees hold in expectation, one can de-randomize choosing a single $h_t$ uniformly at random as above. Moreover, the confidence parameter of $1/2$ can easily be made arbitrarily large using standard confidence-boosting techniques. That is, by first splitting the input sample into approximately $q=O(\log(r/\delta))$ non-intersecting equal parts to learn roughly $q$ independent classifiers, each satisfying Equation \eqref{eq:boosted_CS_to_MO}. By the choice of $q$, with $1-\delta/2$ probability, at least one of these classifiers will the satisfy the $1/2$-probability condition given in  Equation \eqref{eq:boosted_CS_to_MO}. 
Then, by Chernoff’s inequality, for each of these classifiers the performance as tested over another independent test-set is close to that over the population distribution. Then, by choosing the one with best empirical errors over the test-set, we get the desired result. 
\end{proof}

%% file: 5-multi.tex
\section{Multiclass Classification}\label{sec:multiclass}

In contrast to binary classification, boosting in the multiclass setting does not exhibit a clean dichotomy between ``trivial'' guarantees and ``fully boostable'' guarantees. A line of works \citep{Mukherjee2013, brukhim2021multiclass, brukhim2023improper, Brukhim23simple} has studied multiclass boosting when $\cost$ is the standard 0-1 loss, and has shown a certain \emph{multichotomy} determined by $k-1$ thresholds, $\frac{1}{2}, \frac{2}{3}, \dots, \frac{k-1}{k}$, that replace the single boostability threshold $\nicefrac{1}{2}$ of the binary case.
For each such threshold, any learner achieving a loss below it can be  ``boosted'' to a predictor that rules out \emph{some} incorrect labels, albeit not all. Thus, for every example $x$ the predictor returns a subset of candidate labels, also called a \emph{list}. A learner that beats a threshold $\frac{\ell-1}{\ell}$ above in fact yields a list predictor with lists of size $\ell-1$, and it can be shown that the converse is also true. In this section we generalize these kind of results to the case of arbitrary costs functions, proving that these list boostability and list-versus-loss equivalences hold more broadly.

We start in \Cref{sub:multi_scalar_UB} by describing an algorithm that turns a $(\cost,z)$-learner into a list learner, proving the first part of \Cref{thm:main_multiclass}.
Next, \Cref{subsec:multiclass-list-size} proves some ``list boostability'' results that focus on the list length, based on the framework of {\it List PAC learning} \citep{brukhim2022characterization, charikar2022characterization}.
We turn to multi-objective losses in \Cref{sub:multi_MO}, proving the first item of \Cref{thm:multiclass_MO_boost}.
Finally, \Cref{subsec:multiclass:lower_bounds} concludes by giving the lower bounds of both \Cref{thm:main_multiclass} and \Cref{thm:multiclass_MO_boost}.

\subsection{Cost-sensitive boosting}\label{sub:multi_scalar_UB}
As said, in multiclass classification one cannot in general boost a learner into one that predicts the correct label. However, as prior work has shown, one can boost a learner into a \emph{list learner}. This is a learner whose output hypothesis maps each $x \in \X$ to a set of labels, rather than to one single label. Formally, a \emph{list function} is any map $\mu : \X \to 2^\Y$. A list learner is then a learner that outputs a list function $\mu$. The performance of such a list function $\mu$ is typically measured in two ways. First, one wants $\mu(x)$ to contain the true label $f(x)$ with good probability. Second, one wants $\mu(x)$ to somehow be ``close'' to $f(x)$; for instance, in the case of standard 0-1 cost, this is measured by the length $|\mu(x)|$ of $\mu(x)$. All these guarantees are meant, as usual, over the target distribution $\D$.
One can then ask the following question: given a $(\cost,z)$-learner, how good a list function $\mu$ can one construct?
In order to answer this question in the cost-sensitive setting, we need to generalize list length to a more nuanced measure---one that is based once again on the values of games.
\begin{definition}[$(\cost,z)$-boundedness]\label{def:wz_list} %
Let $\cost: \Y^2 \rightarrow [0,1]$ be a cost function and let $z \ge 0$. A list function $\mu : \X \to 2^\Y$ is \emph{$(\cost,z)$-bounded} if, for every $x \in \X$, the list $\mu(x)$ satisfies $\val_{\!\mu(x)}(\cost) \le z$.
\end{definition}
\begin{definition}[\smash{$(\cost, z)$-list learner}]\label{def:z_list}
Let $\cost: \Y^2 \rightarrow [0,1]$ be a cost function and let $z \ge 0$. An algorithm $\L$ is a \emph{$(\cost, z)$-list learner} for a class $\F \subseteq \Y^\X$ if there is a function $m_\L: (0,1)^2 \rightarrow \mathbb{N}$ such that for every $f \in \F$, every distribution $\D$ over $\X$, and every $\epsilon,\delta \in (0,1)$ the following claim holds. If $S$ is a sample of $m_\L(\epsilon, \delta)$  examples drawn i.i.d.\ from $\D$ and labeled by $f$, then $\L(S)$ returns a $(\cost,z)$-bounded list function $\mu$ such that $\Pr_{x \sim \D}\left[f(x) \notin \mu(x)\right]\leq \epsilon$ with probability at least $1-\delta$.
\end{definition}
\noindent 
The definitions above mirror closely the standard ones used in list learning, where the quality of a list $\mu(x)$ is measured by its length $|\mu(x)|$. The crucial difference is that, now, we are instead using the value of the restricted game $\val_{\!\mu(x)}(\cost)$. To get some intuition why this is the right choice, consider again the case where $\cost$ is the standard 0-1 cost. In that case it is well known that $V_J(\cost)=1-\nicefrac{1}{|J|}$ for every (nonempty) $J \subseteq \Y$.
Then, by \Cref{def:wz_list} and \Cref{def:z_list} a $(\cost,z)$-list learner is one that outputs lists of length at most $\ell = \min(k, \lfloor\frac{1}{1-z}\rfloor)$. That is, for the 0-1 cost \Cref{def:z_list} yields the standard notion of $\ell$-list learner used in prior work on multiclass boosting.
Now, for the special case of the 0-1 cost, \cite{Brukhim23simple} showed that any $(\cost,z)$-learner is equivalent to an $\ell$-list learner, with $\ell$ as defined above, in the sense that each one of them can be converted into the other. We show that this remains true for \emph{every} cost function $\cost$, if one replaces $|\mu(x)|$ with $\val_{\!\mu(x)}(\cost)$, and therefore the notion of $\ell$-list learner with the notion of $(\cost,z)$-list learner. That is, we prove that $(\cost,z)$-learners and $(\cost,z)$-list learners are equivalent. Formally, we prove:
\begin{theorem}\label{thm:multi_wz-list_equivalence}
Let $\Y=[k]$, let $\cost: \Y^2 \rightarrow [0,1]$ be a cost function, let $\F\subseteq \Y^\X$, and let $z \ge 0$. Then:
\begin{itemize}\itemsep0pt
    \item Every $(\cost,z)$-learner for $\F$ can be converted into a $(\cost,z)$-list learner for $\F$.
    \item Every $(\cost,z)$-list learner for $\F$ can be converted into a $(\cost,z)$-learner for $\F$.
\end{itemize}
\end{theorem}
\noindent \Cref{thm:multi_wz-list_equivalence} is the main technical result of the present sub-section, and implies easily the first item of \Cref{thm:main_multiclass}. And again, the above-mentioned result of \cite{Brukhim23simple} is a special case of \Cref{thm:multi_wz-list_equivalence}, obtained for $\cost$ being the standard 0-1 cost.  
Moreover, one can easily show that all these results do not hold if one  uses $|\mu(x)|$ in place of $\val_{\!\mu(x)}(\cost)$.
Our game-theoretical perspective therefore seems the right point of view for characterizing multiclass boostability. 
Let us see how the first item of \Cref{thm:main_multiclass} follows from \Cref{thm:multi_wz-list_equivalence}. Recall the thresholds defined in \Cref{eq:multiclass_critical_thresholds}:
\begin{equation}%
   0 = v_{1}(\cost) < \ldots < v_{\tau}(\cost) = V(\cost),
\end{equation}
where the first equalities hold because $v_J(\cost) = 0$ for any singleton $J$ and, by convention, for the empty set. Given $z \ge 0$, let $n=n(z)$ be the largest integer such that $v_n(\cost) \le z$.
Now, suppose we have a $(\cost,z)$-learner $\A$ for some $\F$. By the first item of \Cref{thm:multi_wz-list_equivalence}, there exists a $(\cost,z)$-list learner $\L$ for $\F$, too. However, by definition of $n$ every $J \subseteq \Y$ with $\val_J(\cost) \le z$ satisfies $\val_J(\cost) \le v_n$. Therefore, $\L$ actually returns list functions that are $(\cost,v_n+\sigma)$-bounded, for $\sigma>0$ arbitrarily small. That is, $\L$ is actually a $(\cost,v_n)$-list learner for $\F$. By the second item of \Cref{thm:multi_wz-list_equivalence}, then, there exists a $(\cost,v_n)$-learner for $\F$.

The rest of the subsection is therefore devoted to prove \Cref{thm:multi_wz-list_equivalence}: the first item in \Cref{sub:wz_to_wzlist}, and the second one in \Cref{sub:wzlist_to_wz}.

\subsubsection{Turning a $(\cost,z)$-learner into a $(\cost,z)$-list learner}\label{sub:wz_to_wzlist}
We prove the first item of \Cref{thm:multi_wz-list_equivalence} by giving an algorithm, \Cref{alg:multi_weak_to_list_boost}.
The algorithm is very similar to the one for the binary case. And, like for that case, we define a certain \emph{margin} that the learner has, in order to bound the number of ``boosting'' rounds of the algorithm. For the multiclass case, however, the definition is in general slightly different.
\begin{definition}[Margin]\label{def:multi_margin}
    Let $\cost: \Y^2 \to [0,1]$ be a cost function and let $z \ge 0$.
    The margin of a $(\cost,z)$-learner is $\gamma = \min\{\val_J(\cost) - z : J \subseteq \Y, \val_J(\cost) > z\}$, or $\gamma = 0$ whenever the set is empty.
\end{definition}
\noindent Note that $\gamma=0$ only when $z \ge \val(\cost)$. In that case, the first item of \Cref{thm:multi_wz-list_equivalence} holds trivially by just considering the trivial list learner that outputs the constant list function $\mu(x)=\Y$. Hence, in what follows we consider always the case $\gamma > 0$.
We shall prove:
\begin{theorem}[\smash{Weak $\Rightarrow$ List Learning}]\label{thm:wl_to_list}
Let $\F \subseteq \Y^\X$ and let $\cost: \Y^2 \to [0,1]$ be a cost function.
If $\A$ is a $(\cost,z)$-learner for $\F$ with margin $\gamma > 0$, then \Cref{alg:multi_weak_to_list_boost}, with the choice of parameters of \Cref{lemma:wl_to_list} and $\sigma=\frac23\gamma$, is a $(\cost,z)$-list learner for $\F$.
Under the same assumptions, \Cref{alg:multi_weak_to_list_boost} makes $T = O\bigl(\frac{\ln(m)}{\gamma^2}\bigr)$ oracle calls to $\A$, where $m$ is the size of the input sample, and has sample complexity $m(\eps,\delta) = \widetilde{O}\Bigl(m_0\bigl(\frac{\gamma}{3},\frac{\delta}{2T}\bigr) \cdot \frac{\ln(1/\delta)}{\eps \cdot \gamma^2}\Bigr)$.
\end{theorem}
\noindent\Cref{thm:wl_to_list} follows immediately from two technical results stated below.
The first result, \Cref{lemma:wl_to_list}, proves that \Cref{alg:multi_weak_to_list_boost} can turn a $(\cost,z)$-learner $\A$ into a list function $\mu_S$ that is $(\cost,z+\sigma)$-bounded and that with probability $1-\delta$ is consistent with $S$ (i.e., $y_i \in \mu(x_i)$ for every $(x_i,y_i) \in S$), for $\sigma,\delta > 0$ as small as desired.
The second result, \Cref{lem:multi_scalar_boost_generalization}, shows that the list function $\mu$ has small generalization error: if the input sample $S$ of \Cref{alg:multi_weak_to_list_boost} is sufficiently large, then $\Pr_{x \sim \D }[f(x) \in \mu(x)] \ge 1-\eps$, where $\eps>0$ can be chosen as small as desired. 
 
\begin{lemma}\label{lemma:wl_to_list}
Let $\Y=[k]$, let $\F \subseteq \Y^\X$, and let $\A$ be a $(\cost, z)$-learner for $\F$ with sample complexity $m_0$.
Fix any $f \in \F$, any distribution $\D$ over $\X$, and any $\sigma,\delta \in (0,1)$. Then, given a multiset $S=\{(x_1,y_1),\ldots,(x_m,y_m)\}$ of $m$ examples formed by i.i.d.\ points $x_1,\dots,x_m \sim \D$ labeled by $f$, oracle access to $\A$, and parameters $T = \ceil*{\frac{8\ln(m)}{\sigma^2}}$, $\eta = \sqrt{\frac{2\ln(m)}{T}}$, and $\widehat m = m_0\bigl(\frac{\sigma}{2},\frac{\delta}{T}\bigr)$, \Cref{alg:multi_weak_to_list_boost} returns a list function $\mu_S$ that is $(\cost,z+\sigma)$-bounded and that, with probability at least $1-\delta$, satisfies:
\[
    y_i \in \mu_S(x_i) \qquad \forall i=1,\dots,m \;.
\]
\end{lemma}
\begin{proof}
    First, we prove that $\mu_S$ is $(\cost,z+\sigma)$-bounded.
    Fix any $x \in \X$. Observe that $F(x,\cdot)$ is a distribution over $\Y$; denote it by $\bp^x \in \Delta_\Y$, so that $p^x_\ell = F(x,\ell)$ for every label $\ell \in \Y$.
    Thus, by the definition of $\mu_S$, the sum involved in the condition for including a certain label $y \in \Y$ in the list $\mu_S(x)$ is
    \begin{equation}
      \sum_{\ell\in \Y} F(x,\ell) \cdot \cost(\ell,y) = \cost(\bp^x,y) \;.
    \end{equation}
    Then, observe that
    \begin{align}
        \val_{\!\mu_S(x)}(\cost) = \min_{\bp \in \Delta_\Y} \max_{\bq \in \Delta_{\mu_S(x)}} \cost(\bp,\bq) \le \max_{\bq \in \Delta_{\mu_S(x)}} \cost(\bp^x,\bq) = \max_{y \in \mu_S(x)} \cost(\bp^x,y) \le  z + \sigma \;,
    \end{align}
    where the last inequality holds by definition of $\mu_S(x)$.

    We now turn to proving that with probability at least $1-\delta$ we have $y_i \in \mu_S(x_i)$ for all $i\in[m]$. The proof is similar to that of \Cref{thm:binary_scalar_boosting}.
    For any given $i \in [m]$, the analysis of Hedge and the definitions of $\eta$ and $T$ yield:
    \begin{equation}\label{eq:multiclass_hedge_bound_finer}
        \frac{1}{T} \sum_{t=1}^T \cost(h_t(x_i),y_i) \le \frac{\sigma}{2} + \frac{1}{T} \sum_{t=1}^T \E_{j \sim \D_t}\Bigl[\cost(h_t(x_j),y_j)\Bigr] \;.
    \end{equation}
    Furthermore, since $\A$ is a $(\cost, z)$-learner for $\F$, and given the choices of $\widehat m$ and $h_t = \A(S_t)$, by a union bound over $t \in [T]$ we have that with probability at least $1 - \delta$:
    \begin{equation}\label{eq:multiclass_avg_loss_bound_finer}
        \frac{1}{T} \sum_{t=1}^T \E_{j \sim \D_t}\Bigl[\cost(h_t(x_j),y_j)\Bigr] \le z + \frac{\sigma}{2} \;.
    \end{equation}
    
    Conditioning on \Cref{eq:multiclass_avg_loss_bound_finer}, we prove that $y_i \in \mu_S(x_i)$ for every $i \in [m]$.
    Consider the function $F: \X \times \Y \to [0,1]$ defined in \Cref{alg:multi_weak_to_list_boost}.
    By \Cref{eq:multiclass_hedge_bound_finer,eq:multiclass_avg_loss_bound_finer}, we obtain for every $i \in [m]$ that
    \begin{equation}
        \sum_{\ell \in \Y} F(x_i, \ell) \cdot \cost(\ell,y_i) = \frac{1}{T} \sum_{t=1}^T \cost(h_t(x_i),y_i) \le z + \sigma \;.
    \end{equation}
    which in turn implies that $y_i \in \mu_S(x_i)$ by construction of $\mu_S$. This concludes the proof.
\end{proof}

\begin{algorithm}[t]
\caption{Boosting a $(\cost,z)$-learner to an $(\cost,z+\sigma)$-list learner} 
\label{alg:multi_weak_to_list_boost}
\begin{algorithmic}[1]
\REQUIRE Sample $S = (x_i,y_i)_{i=1}^m$;  $(\cost, z)$-learner $\A$; parameters $T, \eta, \widehat m, \sigma$
\vskip3pt
\STATE Initialize: $D_1(i) = 1$  for all $i=1,...,m$.
\FOR{$t = 1, \ldots, T$}
\STATE Compute the distribution $\D_t \triangleq \frac{D_t}{\sum_{i}D_t(i)}$ over $S$
\STATE Draw a set $S_t$ of $\widehat m$ labeled examples i.i.d.\ from $\D_t$ and obtain $h_t = \A(S_t)$.
\STATE For every $i=1,\ldots,m$ let:
\[
D_{t+1}(i) \triangleq D_t(i) \cdot e^{\eta \cdot \cost(h_t(x_i),y_i)} \;.
\]
\ENDFOR
\STATE Let $F: \X\times\Y \to [0,1]$ be such that $F(x,y) \triangleq \frac{1}{T}\sum_{t=1}^T \ind{h_t(x) = y}$ for all $x \in \X$, $y \in \Y$.
\RETURN $\mu_S : \X \to 2^\Y$ such that, for all $x \in \X$, 
\[
\mu_S(x) \triangleq \Biggl\{y \in \Y : \sum_{\ell\in \Y} F(x,\ell) \cdot \cost(\ell,y) \le z + \sigma \Biggr\} \;.
\]
\end{algorithmic}
\end{algorithm}

In a similar way as \Cref{lem:bin_scalar_boost_generalization} does for the binary case, we show that the list function $\mu_S$ returned by \Cref{alg:multi_weak_to_list_boost} has small generalization error provided that the sample $S$ is sufficiently large.
The main idea to prove generalization is again via a sample compression scheme, but relying instead on a novel result by \citet{Brukhim23simple} for multiclass classification; note that, while their result from \citet{Brukhim23simple} is stated in terms of list functions whose outputs are $s$-uples in $\Y^s$ for some $s < k$, their same proof applies to the $(\cost,z)$-bounded list functions output by our \Cref{alg:multi_weak_to_list_boost}.

\begin{lemma}\label{lem:multi_scalar_boost_generalization}
        Assume the setting of \Cref{lemma:wl_to_list}.
    For any $\eps,\delta \in (0,1)$, if the size $m$ of the sample given to \Cref{alg:multi_weak_to_list_boost} satisfies
    \[
        m \ge \frac{\ln(2/\delta)}{\eps} + m_0\left(\frac{\sigma}{2}, \frac{\delta}{2T}\right) \cdot T \cdot \left(1 + \frac{\ln(m)}{\eps}\right) \;,
    \]
    then the output $\mu_S$ of \Cref{alg:multi_weak_to_list_boost} satisfies $\Pr_{x \sim \D}[f(x) \notin \mu_S(x)] \le \eps$ with probability at least $1-\delta$.
    Therefore, \Cref{alg:multi_weak_to_list_boost} is a $(\cost, z+\sigma)$-list learner for $\F$ with sample complexity $m(\eps,\delta) = \widetilde{O}\Bigl(m_0\bigl(\frac{\sigma}{2},\frac{\delta}{2T}\bigr) \cdot \frac{\ln(1/\delta)}{\eps\cdot\sigma^2}\Bigr)$.
    \end{lemma}
\begin{proof}
    By analogy with the proof of \Cref{lem:bin_scalar_boost_generalization}, we first apply \Cref{lemma:wl_to_list} with $\delta/2$ in place of $\delta$ in order to obtain an $(\cost,z+\sigma)$-bounded list function $\mu_S$ consistent with $S$ with probability at least $1-\delta/2$.
    We then apply a compression-based generalization argument.
    To do so, we remark once again that one can construct a compression scheme for $\mu_S$ of size equal to the total size of the samples on which $\A$ operates, which is equal to $\kappa = \widehat m \cdot T$. %
    The main difference with the binary case is that we rely on the generalization bound for a sample compression scheme for lists as per Theorem~6 of \citet{Brukhim23simple}, with $\delta/2$ in place of $\delta$;
    we can employ this generalization bound thanks to the consistency of $\mu_S$ with $S$ (with sufficiently large probability).
    Then, this implies that
    \begin{equation}
        \Pr_{x \sim \D}\bigl[f(x) \notin \mu_S(x)\bigr] \le \frac{\kappa \ln(m) + \ln(2/\delta)}{m - \kappa}
    \end{equation}
    holds with probability at least $1-\delta/2$.
    By similar calculations as in \Cref{eq:m_ge_m0_0,eq:m_ge_m0}, and replacing the values of $\kappa$ and $\widehat m$, the right-hand side of the above inequality can be easily show to be at most $\eps$ whenever
    \begin{equation}\label{eq:m_ge_m0_multi}
        m \ge \frac{\ln(2/\delta)}{\eps} + m_0\left(\frac{\sigma}{2}, \frac{\delta}{2T}\right) \cdot T \cdot \left(1 + \frac{\ln(m)}{\eps}\right) \;.
    \end{equation}
    A union bound concludes the proof for the first part of the claim, since it shows that \Cref{alg:multi_weak_to_list_boost} is an $(\cost,z+\sigma)$-list learner for $\F$.
    The claim on the sample complexity of \Cref{alg:multi_weak_to_list_boost} then follows by straightforward calculations and the definition of $T = O\bigl(\frac{\ln(m)}{\sigma^2}\bigr)$.
    \end{proof}

At this point, we have all the ingredients and tools to prove that we can construct a $(\cost,z)$-list learner from a $(\cost,z)$-learner. 
\begin{proof}[Proof of \Cref{thm:wl_to_list}]
    Consider \Cref{alg:multi_weak_to_list_boost} under the same assumptions of \Cref{lemma:wl_to_list}, and set $\sigma = \frac23\gamma$ as per assumption.
    By \Cref{lemma:wl_to_list} and \Cref{lem:multi_scalar_boost_generalization}, we know that \Cref{alg:multi_weak_to_list_boost} is a $(\cost,z+\frac23\gamma)$-list learner with sample complexity $m(\eps,\delta) = \widetilde{O}\bigl(m_0\bigl(\frac{\gamma}{3},\frac{\delta}{2T}\bigr) \cdot\frac{\ln(1/\delta)}{\eps\cdot\gamma^2}\bigr)$ that performs $O\bigl(\frac{\ln(m)}{\gamma^2}\bigr)$ oracle calls to the $(\cost,z)$-learner $\A$.
    Moreover, we can immediately conclude that \Cref{alg:multi_weak_to_list_boost} is also a $(\cost,z)$-list learner because $z \ge v_n$ and $z+\frac23\gamma < z + \gamma = v_{n+1}$, by definitions of $n=n(z)$ and $\gamma$.
\end{proof}

\subsubsection{Turning a $(\cost,z)$-list learner into a $(\cost,z)$-learner}\label{sub:wzlist_to_wz}

We prove that every $(\cost,z)$-list learner can be converted into a $(\cost,z)$-learner. 

\begin{lemma}[\smash{List $\Rightarrow$ Weak Learning}]\label{lemma:list_to_wl}
There exists an algorithm $\B$ that for every $\F \subseteq \Y^\X$ satisfies what follows. Let $\cost: \Y^2 \rightarrow [0,1]$ be a cost function and let $z \ge 0$. Given oracle access to a $(\cost,z)$-list learner $\L$ for $\F$ with sample complexity $m_{\L}$, algorithm $\B$ is a $(\cost, z)$-learner for $\F$ with sample complexity $m_\L$ that makes one oracle call to $\L$.
\end{lemma}
\begin{proof}
Fix any $f \in \F$ and any distribution $\D$ over $\X$, and suppose $\B$ is given a sample $S$ of size $m \ge m_\L(\epsilon,\delta)$ consisting of examples drawn i.i.d.\ from $\D$ and labeled by $f$. First, $\B$ calls $\L$ on $S$. By hypothesis, then, $\L$ returns a $(\cost,z)$-bounded $\mu_S: \X \to \Y^\X$ that with probability at least $1 - \delta$ satisfies:
\begin{equation}
    \Pr_{x \sim \D}\bigl[f(x) \notin \mu_S(x)\bigr] \le \epsilon \;.\label{eq:pr_fx_ne_mu_finer}
\end{equation}
Conditioning on the event above from this point onwards, we give a randomized predictor $h_S$ such that $L_{\D}^\cost(h_S) \le z + \eps$.
First, for any nonempty $J \subseteq \Y$, let $\bp_J \in \Delta_\Y$ be the minimax distribution achieving the value of the game restricted to $J$, i.e.,
\begin{align}
    \bp_J = \arg \min_{\bp \in \Delta_\Y} \left( \max_{\bq \in \Delta_J} \cost(\bp,\bq) \right) \;.
\end{align}
Note that $\bp_J$ can be efficiently computed via a linear program.
Then, simply define $h_S(x) \sim \bp_{\mu_S(x)}$.

Let us analyse the loss of $h_S$ under $\D$. First, by the law of total expectation, and since $\norm{\cost}_{\infty} \le 1$,
\begin{align}
    L_{\D}^\cost(h_S) &\le \Pr_{x \sim D}[f(x) \notin \mu_S(x)] + \sum_{J} \Pr_{x \sim \D}[\mu_S(x)=J \wedge f(x) \in J] \cdot L_{\D_J}^\cost(h_S) \;,\label{eq:LDcost_hS_finer}
\end{align}
where the summation is over all $J \subseteq \Y$ with $\Pr_{x\sim \D}[\mu_S(x)=J]>0$, and $\D_J$ is the distribution obtained from $\D$ by conditioning on the event $\mu_S(x)=J \wedge f(x) \in J$.
Consider the right-hand side of \Cref{eq:LDcost_hS_finer}. By \Cref{eq:pr_fx_ne_mu_finer}, the first term is at most $\eps$. For the second term, denote by $\bq_J$ the marginal of $\D_J$ over $\Y$; note that, crucially, $\bq_J \in \Delta_J$. Therefore, by definition of $h_S$:
\begin{align}
    L_{\D_J}^\cost(h_S) = \cost(\bp_J,\bq_J) \le \max_{\bq \in \Delta_J} \cost(\bp_J,\bq) = \val_J(\cost) \le z \;,\label{eq:LDJ_z_finer}
\end{align}
where the last inequality holds as $J=\mu(x)$ and $\mu$ is $(\cost,z)$-bounded. Using \Cref{eq:pr_fx_ne_mu_finer,eq:LDJ_z_finer} in \Cref{eq:LDcost_hS_finer} shows that $L_{\D}^\cost(h_S) \le z + \eps$.
\end{proof}

\input{5.2-multi-MO}
\input{5.3-multi-list-sizes}

\input{5.4-lower-bounds}

%% file: 5.2-multi-MO.tex
\subsection{Multi-objective losses}\label{sub:multi_MO}
In this sub-section we prove the first item of \Cref{thm:multiclass_MO_boost}. Coupled with \Cref{lemma:LB__multiclass_MO_boost} below, this proves \Cref{thm:multiclass_MO_boost}.
Both results require the following definition. Let $\excl(\costVec,\bz)$ denote the family of all subsets of $\Y$ which are {\it avoided} by $(\costVec,\bz)$, i.e., 
$$
\excl(\costVec,\bz) \triangleq \{ J \subseteq \Y : \bz \notin D_J(\costVec)\}. 
$$
The algorithm proposed below then scales with $N= |\excl(\costVec,\bz)|$, which may be exponential in $|\Y|$. We remark that for our purposes, it also suffices to consider only the {\it minimally-sized} avoided lists in $\excl(\costVec,\bz)$, i.e., if $J \subseteq J'$ and $\bz \notin D_J(\costVec)$ then $J' \notin \excl(\costVec,\bz)$. However, in the worst case $N$ remains of the same order and so we keep the definition as given above for simplicity. \Cref{lemma:LB__multiclass_MO_boost} will then give a lower bound showing that this condition is, in some sense, necessary.

\begin{theorem}\label{thm:UB_multiclass_MO_boost}
    Let $\Y=[k]$, let $\costVec = (\cost_1,\ldots,\cost_r)$ where each $\cost_i: \Y^2 \rightarrow [0,1]$ is a cost function, and let $\bbz \in [0,1]^r$. Let $\bbz'  \in [0,1]^r$ such that $\bbz \preceq_{\costVec} \bbz'$.  Assume there exists a $(\costVec, \bbz)$-learner for a class $\F \subseteq \Y^\X$. Then, there exists a $(\costVec, \bbz')$-learner for $\F$. 
\end{theorem}
\begin{proof}
By assumption, for every $J \subseteq \Y$ such that $J \in \excl(\costVec,\bz')$ it holds that  $J \in \excl(\costVec,\bz)$.  Let $\balpha \in \Delta_r$, and denote $\cost_{\balpha} = \balpha \cdot \costVec$ and $z'_{\balpha} = \balpha \cdot \bz'$. Then, by \Cref{lemma:MCMO_to_MCCS} we get that the $(\costVec, \bbz)$-learner can be used to obtain a $(\cost_{\balpha}, z'_{\balpha})$-learner for $\F$. Since this holds for every $\balpha  \in \Delta_r$, then by \Cref{thm:intro_general_duality} this implies the existence of a $(\costVec,\bz')$-learner as well, which completes the proof.
\end{proof}

\begin{lemma}\label{lemma:MCMO_to_MCCS}
        Let $\costVec = (\cost_1,\ldots,\cost_r)$ where each $\cost_i: \Y^2 \rightarrow [0,1]$ is a cost function, and let $\bbz, \bbz' \in [0,1]^r$  such that $\bbz \preceq_{\costVec} \bbz'$. Assume there exists a $(\costVec, \bbz)$-learner $\A$ for a class $\F \subseteq \Y^\X$. Then, there exists a $(\cost_{\balpha}, z'_{\balpha})$-learner for $\F$, for every $\balpha \in \Delta_r$.
\end{lemma}
\begin{proof}
 First, by \Cref{lemma:MCMO_to_LL} we get that $\A$ can be used to obtain an algorithm  $\L$ that is a $\excl(\costVec, \bbz')$-list learner for $\F$. 
Next, we show that $\L$ can be boosted to a $(\cost_{\balpha},z'_{\balpha})$-learner. Fix any $\balpha \in \Delta_r$. First, we argue that $\L$ is in fact a $(\cost_{\balpha}, z'_{\balpha})$-list learner (see 
\Cref{def:z_list}). This holds since for every $J \subseteq \Y$ such that $z'_{\balpha} < V_J(\cost_{\balpha})$, it must be that $J \in  \excl(\costVec, \bbz')$, by definition and by \Cref{prop:J_dice_duality}. In particular, for any $\mu$ list function outputted by $\L$, and for every $x \in \X$ it holds that $J \not\subseteq \mu(x)$ and so $V_{\mu(x)}(\cost_{\balpha}) \le z'_{\balpha}$. 

Thus, we have established that $\L$ is in fact a $(\cost_{\balpha}, z'_{\balpha})$-list learner. 
Lastly, by \Cref{lemma:list_to_wl} we get that it can also be converted into a $(\cost_{\balpha}, z'_{\balpha})$-learner, as needed.
\end{proof}

\begin{definition}[\smash{$\excl(\costVec, \bbz)$-list learner}] 
An algorithm $\L$ is a \emph{$\excl(\costVec, \bbz)$-list learner} for a class $\F \subseteq \Y^\X$ if there is a function $m_\L: (0,1)^2 \rightarrow \mathbb{N}$ such that for every $f \in \F$, every distribution $\D$ over $\X$, and every $\epsilon,\delta \in (0,1)$ the following claim holds. If $S$ is a sample of $m_\L(\epsilon, \delta)$  examples drawn i.i.d.\ from $\D$ and labeled by $f$, then $\L(S)$ returns a list function $\mu: \X \rightarrow 2^\Y$ such that 
$\Pr_{x \sim \D}\left[f(x) \notin \mu(x)\right]\leq \epsilon$ with probability   $1-\delta$, and such that for every $x \in \X$ and every $J \in \excl(\costVec, \bbz)$ it holds that $J \not\subseteq \mu(x)$.
\end{definition}
\noindent 

\begin{lemma}\label{lemma:MCMO_to_LL}
        Let $\Y=[k]$, let $\costVec = (\cost_1,\ldots,\cost_r)$ where each $\cost_i: \Y^2 \rightarrow [0,1]$ is a cost function, and let $\bbz, \bbz' \in [0,1]^r$  such that $\bbz \preceq_{\costVec} \bbz'$.  Assume there exists a $(\costVec, \bbz)$-learner $\A$ for a class $\F \subseteq \Y^\X$. Then, there exists a $\excl(\costVec, \bbz')$-list learner for $\F$. 
\end{lemma}
\begin{proof}
First, fix any $J \in \excl(\costVec, \bbz')$. By assumption that $\bbz \preceq_{\costVec} \bbz'$, we also have that  $J \in \excl(\costVec, \bbz)$. In particular, we have that $(\costVec, \bbz)$ is not $J$-dice-attainable. Then, by \Cref{prop:J_dice_duality} we have that there exist $\balpha \in \Delta_r$ such that, 
\begin{equation}\label{eq:some_ineq_in_some_lemma}
    z_{\balpha} = \langle\balpha,\bbz\rangle <  \val_{J}(w_{\balpha}),
\end{equation}
where $w_{\balpha} = \sum_{i=1}^r \alpha_i w_i$. Thus, we have that $\A$ is also a $(w_{\balpha}, z_{\balpha})$-learner for $\F$, and that $z_{\balpha} <  \val_{J}(w_{\balpha})$. We denote this learning algorithm by $\A_J$. 
Notice that we can repeat the above process for any such $J$.  Thus, we  obtain different learning algorithms $\A_J$ for each $J \in \excl(\costVec, \bbz')$.\\

Next, we will describe the construction of the $\excl(\costVec, \bbz')$-list learning algorithm. 
Fix any $\delta',\epsilon' > 0$.  
For every learner $\A_J$, apply  \Cref{alg:multi_weak_to_list_boost} with parameters $\sigma_{J} = \frac23\gamma_{J}$, where $\gamma_{J}$ is the margin of $\A_J$, our $(w_{\balpha_J}, z_{\balpha_J})$-learner (see \Cref{def:multi_margin}). We also set the parameters of $\A_J$ to be $\epsilon = \epsilon'$, and $\delta = \delta'/2N$, for $N = |\excl(\costVec, \bbz')|$.  The remaining parameters for \Cref{alg:multi_weak_to_list_boost} are set as in \Cref{lemma:wl_to_list}. 

For each such run of the algorithm, we obtain a list function.
Let $\mu_{J_1},...,\mu_{J_N}$ denote all list functions obtained by this procedure. Finally, return the list function $\mu$ defined by $$
\mu(x)=\bigcap_{n=1}^N \mu_{J_n}(x),
$$ for every $x \in \X$. We will now show that this algorithm satisfies the desired guarantees, and is indeed a $\excl(\costVec, \bbz')$-list learner. \\

First, by \Cref{lemma:wl_to_list} and union bound we get that with probability at least $1-\delta/2$, for each $n\le N$ it holds that:
$$
y_i \in \mu_{J_n}(x_i) \qquad \forall i=1,\dots,m \;.
$$
Thus, in particular, with probability at least $1-\delta/2$ we also have that $y_i \in \mu(x_i)$ for all $i \in [m]$.
Moreover, by \Cref{lemma:wl_to_list} we have that all $\mu_{J_n}$ are $(w_{\balpha}, z_{\balpha} + \sigma_{J_n})$-bounded. That is, 
for each $n\le N$ and every $x\in \X$ it holds that $\val_{\!\mu_{J_n}(x)}(w_{\balpha}) \le z_{\balpha} + \sigma_{J_n} = z_{\balpha} + \frac23\gamma_{J_n}$. Now recall that by the definition of the margin, and by \Cref{eq:some_ineq_in_some_lemma} we in fact have that:
$$
\val_{\!\mu_{J_n}(x)}(w_{\balpha}) < \val_{\!J_n}(w_{\balpha}).
$$
This then implies that for every $x \in \X$ it holds that $J_n \not\subseteq \mu_{J_n}(x)$.
Thus, in particular,  we also get that the final list function $\mu$ satisfies that for every $x \in \X$ it holds that $J_n \not\subseteq \mu(x)$. Lastly, by following the same compression-based generalization analysis as in \Cref{lem:multi_scalar_boost_generalization}, we obtain that the above procedure is in fact a $\excl(\costVec, \bbz')$-list learning algorithm, with sample complexity  $m(\eps,\delta) = \widetilde{O}\Bigl(m_0\bigl({\gamma^*},\frac{\delta}{NT}\bigr) \cdot \frac{\ln(N/\delta)}{\eps\cdot{\gamma^*}^2}\Bigr)$, where $\gamma^* = \min_{n \in [N]}\gamma_{J_n}$. 

\end{proof}

%% file: 5.3-multi-list-sizes.tex
\subsection{Multiclass boosting via $\listSize$-list PAC learning}\label{subsec:multiclass-list-size}
In this section we aim to convert weak learners to list learners, with a fixed list size. This is, in some sense, a coarsening of the previous subsection, which also subsumes previous work by \cite{Brukhim23simple}, where the authors demonstrate $k-1$ thresholds which determine the achievable list sizes.
To that end, for every $\listSize=2,\ldots,k$ define the following coarsening of $\listSize$-the critical thresholds of $\cost$:
\begin{align}\label{eq:multiclass_critical_thresholds_min_max}
   v_{\underline{\listSize}}(\cost) \triangleq \min \left\{ \val_J(\cost) : J \in {[k] \choose \listSize}\right\}, \qquad  v_{\overline{\listSize}}(\cost) \triangleq \max \left\{ \val_J(\cost) : J \in {[k] \choose \listSize}\right\} \;.
\end{align}
Clearly, for all $\listSize \in \{2,...,k\}$ we have that $v_{\underline{\listSize}}(\cost) \le v_{\overline{\listSize}}(\cost)$.
It is also easy to see that $0 \le v_{\underline{2}}(\cost) \le \ldots \le v_{\underline{k}}(\cost)$, and $0 \le v_{\overline{2}}(\cost) \le \ldots \le v_{\overline{k}}(\cost)$. When clear from context, we omit $\cost$ and denote the thresholds by $v_{\underline{\listSize}}$ and $v_{\overline{\listSize}}$. We remark that the two types of thresholds $v_{\underline{\listSize}}$  and $v_{\overline{\listSize}}$ are useful for different ways of boosting as specified in the remainder of this section.

Before stating the theorem, we need to first describe what it means to ``partially'' boost. The results given in this section are based on the framework of \emph{List PAC learning} \citep{brukhim2022characterization, charikar2022characterization}. The connection between 
list learning and multiclass boosting was demonstrated by prior works \citep{Brukhim23simple, brukhim2023improper}. Here we strengthen this link and
generalize previous results to hold for arbitrary costs.  We start with introducing list learning in Definition \ref{def:list_pac}, followed by the statement of \Cref{thm:main_multiclass_list_sizes}.

\begin{definition}[\smash{$\listSize$-List PAC Learning \citep{brukhim2022characterization, charikar2022characterization}}]\label{def:list_pac}
An algorithm $\L$ is a $\listSize$-list PAC learner for a class $\F \subseteq \Y^\X$ if the following holds. 
For every distribution $\D$ over $\X$ and target function $f \in \F$, and every $\epsilon, \delta > 0$, if $S$ is a set of $m \ge m(\epsilon,\delta)$ examples drawn i.i.d.\ from $\D$ and labeled by $f$, then $\L(S)$ returns  $\mu: \X \rightarrow \Y^{\listSize}$ such that, with probability at least $1-\delta$,
$$
L_\D(\mu) \triangleq \Pr_{x \sim \D}\bigl[f(x) \notin \mu(x)\bigr]\leq \epsilon \;.
$$
\end{definition}

\begin{theorem}\label{thm:main_multiclass_list_sizes}
Let $\Y=[k]$, let $\cost: \Y^2 \rightarrow [0,1]$ be any cost function, and let $z \ge 0$. Let $v_{\underline{2}} \le \ldots \le v_{\underline{k}}$ as defined in \Cref{eq:multiclass_critical_thresholds_min_max}, and let $\listSize < k$ be the smallest integer for which $z < v_{\underline{\listSize+1}}$ or, if none exists, let $\listSize=k$. Then, the following claims both hold.
\begin{itemize}[leftmargin=.5cm]\itemsep0pt
    \item  {\bf\small $(\cost, z)$ is $\listSize$-boostable:} for every class $\F \subseteq \Y^\X$, every $(\cost,z)$-learner is boostable to an $\listSize$-list PAC learner. 
    \item  {\bf\small $(\cost, z)$ is not $(\listSize-1)$-boostable:} for some $\F  \subseteq \Y^\X$ there exists a $(\cost,z)$-learner that cannot be boosted to an $\listSize'$-list PAC learner with $\listSize' <  \listSize$.
\end{itemize} 
\end{theorem}
\noindent In words, Theorem \ref{thm:main_multiclass_list_sizes} implies that there is a partition of the loss values interval $[0,1]$ into $k$ sub-intervals or ``buckets'', based on $v_{\underline{\listSize}}$. Then, any loss value $z$ in a certain bucket $v_{\underline{\listSize}}$ can be boosted to a list of size $\listSize$, but not below that. 
In \Cref{thm:list_to_wl_SIZES} we show the converse: that any $\listSize$-list learner can be boosted to $(\cost, z)$-learner where $z$ may be the lowest value within the $\listSize$-th interval $v_{\overline{\listSize}}$, but not below it.

The first item of \Cref{thm:main_multiclass_list_sizes} holds trivially for $z \ge v_{\underline{k}}$, since in that case $s=k$ and a $k$-list learner always exists. The case $z < v_{\underline{k}}$ is addressed by \Cref{lemma:wl_to_list_SIZES} below.

\begin{lemma}\label{lemma:wl_to_list_SIZES}
Let $\Y=[k]$, let $\cost: \Y^2 \rightarrow [0,1]$ be a cost function, and define $v_{\underline{2}} \le \ldots \le v_{\underline{k}}$ as in \Cref{eq:multiclass_critical_thresholds_min_max}. Let $\F \subseteq \Y^\X$ and let $\A$ be a $(\cost,z)$-learner for $\F$ with sample complexity $m_0$ such that $z < v_{\underline{k}}$.
Let $\listSize$ be the smallest integer in $\{1,\ldots,k-1\}$ such that $z < v_{\underline{\listSize+1}}$ and let $\gamma = v_{\underline{\listSize+1}} - z > 0$ be the margin.
Fix any $f \in \F$, let $\D$ be any distribution over $\X$, and let $S = \{(x_1,y_1), \dots, (x_m,y_m)\}$ be a multiset of $m$ examples given by i.i.d.\ points $x_1,\dots,x_m \sim \D$ labeled by $f$.
Finally, fix any $\delta \in (0,1)$.
If \Cref{alg:multi_weak_to_list_boost} is given $S$, oracle access to $\A$, and parameters $T = \ceil*{\frac{18\ln(m)}{\gamma^2}}$, $\eta = \sqrt{\frac{2\ln(m)}{T}}$, $\widehat m = m_0\bigl(\frac{\gamma}{3},\frac{\delta}{T}\bigr)$, and $\sigma = \frac23\gamma$, then \Cref{alg:multi_weak_to_list_boost} makes $T$ calls to $\A$ and returns $\mu_S: \X \to \Y^\listSize$ such that with probability at least $1-\delta$:
\[
    y_i \in \mu_S(x_i) \qquad \forall i=1,\dots,m \;.
\]
\end{lemma}
\begin{proof}
    Fix any $f \in \F$ and any $\epsilon, \delta \in (0,1)$. Let $\D$ be any distribution over $\X$ and let $S = \{(x_1,y_1), \dots, (x_m,y_m)\}$ be a multiset of $m$ examples given by i.i.d.\ points $x_1,\dots,x_m \sim \D$ labeled by $f$.
    The first part of this proof follows similar steps as the one of \Cref{thm:binary_scalar_boosting}.
    In particular, fix any $i \in [m]$ and observe that, again by the regret analysis of Hedge and by the definitions of $\eta$ and $T$, \Cref{alg:multi_weak_to_list_boost} guarantees
    \begin{equation}\label{eq:multiclass_hedge_bound}
        \frac{1}{T} \sum_{t=1}^T \cost(h_t(x_i),y_i) \le \frac{\gamma}{3} + \frac{1}{T} \sum_{t=1}^T \E_{j \sim \D_t}\Bigl[\cost(h_t(x_j),y_j)\Bigr] \;.
    \end{equation}
    Furthermore, since $\A$ is a $(\cost, z)$-learner for $\F$, and given the choices of $\widehat m$ and $h_t = \A(S_t)$, by a union bound over $t \in [T]$ we obtain that
    \begin{equation}\label{eq:multiclass_avg_loss_bound}
        \frac{1}{T} \sum_{t=1}^T \E_{j \sim \D_t}\Bigl[\cost(h_t(x_j),y_j)\Bigr] \le z + \frac{\gamma}{3}
    \end{equation}
    with probability at least $1 - \delta$.
    
    Now, we show that the list predictor $\mu_S$ built by \Cref{alg:multi_weak_to_list_boost} is consistent with $S$ with sufficiently large probability.
    Precisely, by conditioning on the event given by the inequality in \Cref{eq:multiclass_avg_loss_bound}, we now demonstrate that $y_i \in \mu_S(x_i)$ for every $i \in [m]$.
    Consider the function $F: \X \times \Y \to [0,1]$ as defined by \Cref{alg:multi_weak_to_list_boost}.
    Then, by \Cref{eq:multiclass_hedge_bound,eq:multiclass_avg_loss_bound} together with the definition of $\gamma > 0$, we obtain for every $i \in [m]$ that
    \begin{equation}
        \sum_{\ell \in \Y} F(x_i, \ell) \cdot \cost(\ell,y_i) = \frac{1}{T} \sum_{t=1}^T \cost(h_t(x_i),y_i) \le z + \frac{2}{3}\gamma < z + \gamma = v_{\underline{\listSize+1}}\;,
    \end{equation}
    which in turn implies that $y_i \in \mu_S(x_i)$ by construction of $\mu_S$.
    
    We now proceed with a deterministic characterization of the lists returned by $\mu_S$, independently of any conditioning.
    Fix any $x \in \X$.
    Let $\bp^x \in \Delta_\Y$ be such that $p^x_\ell = F(x,\ell)$ for any $\ell \in \Y$, and observe that the sum involved within the condition for $y \in \Y$ in the construction of the list $\mu_S(x)$ is equal to $\cost(\bp^x,y)$.
    Furthermore, it must be true that $\listSize < k$ since $\gamma > 0$.
    We can then show that the list $\mu_S(x)$ satisfies
    \begin{equation}
        v_{\underline{\listSize+1}}
        > \max_{y \in \mu_S(x)} \cost(\bp^x,y)
        = \max_{\bq \in \Delta_{\mu_S(x)}} \cost(\bp^x,\bq)
        \ge \min_{\bp \in \Delta_\Y} \max_{\bq \in \Delta_{\mu_S(x)}} \cost(\bp,\bq)
        = \val_{\mu_S(x)}(\cost) \;,
    \end{equation}
    where the first inequality holds by definition of $\mu_S(x)$.
    Consequently, after observing that $\val_J(\cost) \le \val_{J'}(\cost)$ if $J \subseteq J'$, note that
    \begin{equation}
        \val_{\mu_S(x)}(\cost) < v_{\underline{\listSize+1}} = \min_{J \in \binom{\Y}{\listSize+1}} \val_J(\cost) = \min_{J \subseteq \Y: \abs{J} \ge \listSize+1} \val_J(\cost) \;,
    \end{equation}
    which in turn implies that $\abs{\mu_S(x)} \le \listSize$. We can thus assume that $\mu_S$ outputs elements of $\Y^s$ without loss of generality.
\end{proof}

The following lemma proves that, in a similar way as \Cref{lem:bin_scalar_boost_generalization} for the binary setting, the list predictor output by \Cref{alg:multi_weak_to_list_boost} has a sufficiently small generalization error provided that the sample $S$ has size $m$ sufficiently large.
The main idea to prove generalization is again via a sample compression scheme, but relying instead of novel results by \citet{Brukhim23simple} for multiclass classification.
This generalization result, combined with \Cref{lemma:wl_to_list_SIZES}, suffices to prove \Cref{thm:main_multiclass_list_sizes}.
\begin{lemma}\label{lem:multi_scalar_boost_generalization_SIZES}
    Assume the hypotheses of \Cref{lemma:wl_to_list_SIZES}.
    For any $\eps,\delta \in (0,1)$, if the size $m$ of the sample given to \Cref{alg:multi_weak_to_list_boost} satisfies
    \[
        m \ge \frac{\ln(2/\delta)}{\eps} + m_0\left(\frac{\gamma}{3}, \frac{\delta}{2T}\right) \cdot T \cdot \left(1 + \frac{\ln(m)}{\eps}\right) \;,
    \]
    then with probability at least $1-\delta$ the output $\mu_S$ of \Cref{alg:multi_weak_to_list_boost} satisfies $L_{\D}(\mu_S) \le \eps$.
    Therefore, \Cref{alg:multi_weak_to_list_boost} is an $\listSize$-list PAC learner for $\F$.
    Moreover, one can ensure that \Cref{alg:multi_weak_to_list_boost} makes $O\bigl(\frac{\ln(m)}{\gamma^2}\bigr)$ calls to $\A$ and has sample complexity $m(\eps,\delta) = \widetilde{O}\left(m_0\left(\frac{\gamma}{3}, \frac{\delta}{2T}\right) \cdot \frac{\ln(1/\delta)}{\eps \cdot \gamma^2}\right)$.
\end{lemma}
\begin{proof}
    By analogy with the proof of \Cref{lem:bin_scalar_boost_generalization}, we first apply \Cref{lemma:wl_to_list_SIZES} with $\delta/2$ in place of $\delta$ in order to obtain an $\listSize$-list predictor $\mu_S$ consistent with $S$ with probability at least $1-\delta/2$.
    We then apply a compression-based generalization argument.
    To do so, we remark once again that one can construct a compression scheme for $\mu_S$ of size equal to the total size of the samples on which $\A$ operates, which is equal to $\kappa = \widehat m \cdot T$. %
    The main difference with the binary case is that we rely on the generalization bound for a sample compression scheme for lists as per Theorem~6 of \citet{Brukhim23simple}, with $\delta/2$ in place of $\delta$;
    we can employ this generalization bound thanks to the consistency of $\mu_S$ with $S$ (with sufficiently large probability).
    Then, this implies that
    \begin{equation}
        L_{\D}(\mu_S) = \Pr_{x \sim \D}\bigl[f(x) \notin \mu_S(x)\bigr] \le \frac{\kappa \ln(m) + \ln(2/\delta)}{m - \kappa}
    \end{equation}
    holds with probability at least $1-\delta/2$.
    By similar calculations as in \Cref{eq:m_ge_m0_0,eq:m_ge_m0}, and replacing the values of $\kappa$ and $\widehat m$, the right-hand side of the above inequality can be easily show to be at most $\eps$ whenever:
    \begin{equation}
        m \ge \frac{\ln(2/\delta)}{\eps} + m_0\left(\frac{\gamma}{3}, \frac{\delta}{2T}\right) \cdot T \cdot \left(1 + \frac{\ln(m)}{\eps}\right) \;.
    \end{equation}
    A union bound concludes the proof for the first part of the claim, since it shows that \Cref{alg:multi_weak_to_list_boost} is an $\listSize$-list PAC learner for $\F$.
    The sample complexity of \Cref{alg:multi_weak_to_list_boost} then immediately follows by definition of $T$.
\end{proof}

Next, we prove a kind of converse of \Cref{thm:main_multiclass_list_sizes}. That is, one can convert a list learner to a weak learner. This results is stated formally as follows.
\begin{theorem}[\smash{$\listSize$-List $\Rightarrow$ Weak Learning}]\label{thm:list_to_wl_SIZES}
There exists an algorithm $\B$ that for every $\F \subseteq \Y^\X$ satisfies what follows. Let $\cost: \Y^2 \rightarrow [0,1]$ be a cost function, let $z \ge 0$, and let $v_{\overline{2}} \le \ldots \le v_{\overline{k}}$ as defined in Equation \eqref{eq:multiclass_critical_thresholds_min_max}. Let $\listSize < k$ be the smallest integer for which  $z < v_{\overline{\listSize+1}}$, or if none exists let $\listSize=k$. 
Given oracle access to a $\listSize$-list PAC learner $\L$ for $\F$ with sample complexity $m_{\L}$, algorithm $\B$ is a $(\cost, z)$-learner for $\F$ with sample complexity $m_\L$ that makes one oracle call to $\L$.
\end{theorem}
\begin{proof}
Fix any $f \in \F$ and any distribution $\D$ over $\X$, and suppose $\B$ is given a sample $S$ of size $m \ge m_\L(\epsilon,\delta)$ consisting of examples drawn i.i.d.\ from $\D$ and labeled by $f$. First, $\B$ calls $\L$ on $S$. By hypothesis, then, $\L$ returns $\mu_S: \X \to \Y^\listSize$ that with probability at least $1 - \delta$ satisfies:
\begin{equation}
    \Pr_{x \sim \D}\bigl[f(x) \notin \mu_S(x)\bigr] \le \epsilon \;.\label{eq:pr_fx_ne_mu}
\end{equation}
Conditioning on the event above, we give a randomized predictor $h_S$ such that $L_{\D}^\cost(h_S) \le z + \eps$. First, for any nonempty $J \subseteq \Y$ let $\bp_J \in \Delta_\Y$ be the minimax distribution achieving the value of the game restricted to $J$,
\begin{align}
    \bp_J = \arg \min_{\bp \in \Delta_\Y} \left( \max_{\bq \in \Delta_J} \cost(\bp,\bq) \right) \;.
\end{align}
Note that $\bp_J$ can be efficiently computed via a linear program.
Then, simply define $h_S(x) \sim \bp_{\mu_S(x)}$.

Let us analyse the loss of $h_S$ under $\D$. First, by the law of total probability, and since $\norm{\cost}_{\infty} \le 1$,
\begin{align}
    L_{\D}^\cost(h_S) &\le \Pr_{x \sim D}[f(x) \notin \mu_S(x)] + \sum_{J} \Pr_{x \sim \D}[\mu_S(x)=J \wedge f(x) \in J] \cdot L_{\D_J}^\cost(h_S) \;,\label{eq:LDcost_hS}
\end{align}
where the summation is over all $J \subseteq \Y$ with $\Pr_{x\sim \D}[\mu_S(x)=J]>0$, and $\D_J$ is the distribution obtained from $\D$ by conditioning on the event $\mu_S(x)=J \wedge f(x) \in J$.
Consider the right-hand side of \Cref{eq:LDcost_hS}. By \Cref{eq:pr_fx_ne_mu}, the first term is at most $\eps$. For the second term, denote by $\bq_J$ the marginal of $\D_J$ over $\Y$; note that, crucially, $\bq_J \in \Delta_J$. Therefore, by definition of $h_S$ and of $s$:
\begin{align}
    L_{\D_J}^\cost(h_S) = \cost(\bp_J,\bq_J) \le \max_{\bq \in \Delta_J} \cost(\bp_J,\bq) = \val_J(\cost) \le v_{\overline{s}} \le z \;.
\end{align}
We conclude that $L_{\D}^\cost(h_S) \le z + \eps$.
\end{proof}

%% file: 5.4-lower-bounds.tex
\subsection{Lower bounds}\label{subsec:multiclass:lower_bounds}

We start proving the second item of \Cref{thm:main_multiclass}.
\begin{lemma}\label{lemma:LB__finegrained}
Let $\Y=[k]$ and let $\cost: \Y^2 \rightarrow [0,1]$ be any cost function. Let $0 = v_{1}(\cost) \le v_{2}(\cost) \le \cdots \le v_{\tau}(\cost)$  as defined in \Cref{eq:multiclass_critical_thresholds}, and let $\threshIndex$ be the largest integer for which $v_{\threshIndex}(\cost) \le z$. Then there exists a class $\F \subseteq \Y^\X$ that has a $(\cost,z)$-learner but no $(\cost,z')$-learner for any $z' < v_\threshIndex(\cost)$.
\end{lemma}
\begin{proof}
    If $v_n(\cost)=0$ then the claim is trivial, hence we assume $v_n(\cost)>0$. Observe that this implies that there exists some $J$ with $|J|\ge 2$ such that $v_\threshIndex(\cost)=V_J(\cost)$.
    Let $\bp^* \in \Delta_\Y$ be the distribution achieving $\val_J(\cost)$, that is:
    \begin{equation}
        \bp^* = \argmin_{\bp \in \Delta_\Y} \max_{\bq \in \Delta_J} \cost(\bp, \bq) \;.
    \end{equation}
    Now let $\X$ be any infinite domain (e.g., $\X = \mathbb{N}$) and let $\F = J^\X$. For every distribution $\D$ over $\X$ and every labeling $f:\X\to J$,
    \begin{align}
    L_\D^{\cost}(\bp^*) %
    &= \cost(\bp^*, \bq_\D) \le \max_{\bq \in \Delta_J}\cost(\bp^*, \bq) = \val_J(\cost) = v_n(\cost) \le z \;, 
    \end{align}
    where $\bq_\D \in \Delta_J$ is the marginal of $\D$ over $\Y$. Thus the learner that returns the random guess $h$ based on $\bp^*$ is a $(\cost,z)$-learner for $\F$.
        
    Next, we show that $\F$ has no $(\cost,z')$-learner with $z'<v_n(\cost)$. Suppose indeed towards a contradiction that there exists such a $(\cost,z')$-learner $\A$. By \Cref{thm:wl_to_list}, then, $\F$ admits a $(\cost,z')$-list learner $\L$. Now, as $z'< v_\threshIndex = \val_J(\cost)$, the list function $\mu$ returned by $\L$ ensures $J \not \subseteq \mu(x)$ for all $x \in \X$. Moreover, as $F=J^\X$, we can assume without loss of generality that $\mu(x) \subseteq J$ for all $x \in \X$; otherwise we could remove the elements of $\mu(x) \setminus J$ to obtain a list $\mu'(x)$ such that $\val_{\mu'(x)}(\cost) \le \val_{\mu(x)}(\cost)$ and that $\Pr_{x \sim \D}[f(x) \in \mu'(x)] \ge \Pr_{x \sim \D}[f(x) \in \mu(x)]$. It follows that $\mu(x) \subsetneq J$ for all $x \in \X$. Therefore, $\L$ is a $(|J|-1)$-list PAC learner.
    However, one can verify that the $(|J|-1)$-Daniely-Shalev-Shwartz dimension of $\F$ is unbounded (see Definition~6 of \citet{charikar2022characterization}). It follows by \citet[Theorem~1]{charikar2022characterization} that $\F$ is not $(|J|-1)$-list PAC learnable, yielding a contradiction.
\end{proof}

Next, we prove the second item of \Cref{thm:main_multiclass_list_sizes}.
\begin{lemma}\label{lemma:LB__multiclass_list}
Let $\Y=[k]$, let $\cost: \Y^2 \rightarrow [0,1]$ be any cost function, and let $v_{\underline{2}} \le \ldots \le v_{\underline{k}}$ as defined in Equation \eqref{eq:multiclass_critical_thresholds_min_max}. Let $z \ge v_{\underline{2}}$ and let $\listSize \le k$ be the largest integer for which  $z \ge v_{\underline{\listSize}}$. Then, there exists a class $\F \subseteq \Y^\X$ that has a $(\cost,z)$-learner and is not $(\listSize-1)$-list PAC learnable. 
\end{lemma}
\begin{proof} The proof follows closely the one of \Cref{lemma:LB__finegrained}.
Let $J = \arg\min\left\{\val_{J'}(\cost) : J' \in {\Y \choose \listSize}\right\}$. 
Let $\X$ be any infinite domain (e.g., $\X = \mathbb{N}$) and let $\F = J^\X$. %
As in the proof of \Cref{lemma:LB__finegrained} there is a distribution $\bp^*$ yielding a $(\cost,z)$-learner for $\F$. 
Simultaneously (again from the proof of \Cref{lemma:LB__finegrained}), the class $\F$ is not $(\listSize-1)$-list PAC learnable.  
\end{proof}

Finally, we prove the second item of \Cref{thm:multiclass_MO_boost}.

\begin{lemma}\label{lemma:LB__multiclass_MO_boost}
Let $\Y=[k]$, let $\costVec = (\cost_1,\ldots,\cost_r)$ where each $\cost_i: \Y^2 \rightarrow [0,1]$ is a cost function, and let $\bbz, \bbz' \in [0,1]^r$ such that $\bbz \not\preceq_{\costVec} \bbz'$.  
There exists a class $\F \subseteq \Y^\X$ that admits a $(\costVec, \bbz)$-learner that is trivial, and therefore cannot be boosted to a $(\costVec, \bbz')$-learner.
\end{lemma}
\begin{proof}
By assumption that $\bbz \not\preceq_{\costVec} \bbz'$, there must be a set $J\subseteq \Y$ for which $\bbz  \in D_J(\costVec)$ and $\bbz' \notin D_J(\costVec)$.
Let $\X$ be any infinite domain (e.g., $\X = \mathbb{N}$) and let $\F = J^\X$.
By definition of $D_J(\costVec)$ there exists a trivial learner that is a $(\costVec, \bbz)$-learner for $\F$.
Assume towards contradiction that the trivial learner can be used to construct a $(\costVec, \bbz')$-learner for $\F$. Since $\bbz' \notin D_J(\costVec)$, then by \Cref{prop:J_dice_duality} there exists  $\balpha \in \Delta_r$ such that $\langle\balpha,\bbz'\rangle <  \val_J(w_{\balpha})$, where $w_{\balpha} = \sum_{i=1}^r \alpha_i w_i$.
Then, by \Cref{thm:wl_to_list}, there is a $(w_{\balpha}, z_{\balpha}')$-list learner $\L'$ that outputs list functions $\mu$ such that, for  every $x \in \X$ it holds that:
$$
V_{\mu(x)}(w_{\balpha}) \le z_{\balpha}' < \val_J(w_{\balpha}).
$$
Thus $\L'$ is a $(|J|-1)$-list PAC learner for $\F$.
However, one can verify that the $(|J|-1)$-Daniely-Shalev-Shwartz dimension of $\F$ is unbounded (see Definition~6 of \citet{charikar2022characterization}). It follows by \citet[Theorem~1]{charikar2022characterization} that $\F$ is not $(|J|-1)$-list PAC learnable, yielding a contradiction.
\end{proof}

%% file: app-cost-sensitive.tex
\section{Cost-Sensitive Learning}
\label{app:cost-sensitive}

Cost sensitive learning has a long history in machine learning. The two main motivations are that, first, not all errors have the same impact, and second, that there might be a class-imbalance between the frequency of different classes. 
Over the years there have been many workshops on cost-sensitive learning in ICML (2000) NIPS (2008) SDM (2018), yet another indication of the impostance of cost-sensitive loss.
See, e.g., \cite{LingS17a,ling2008cost}.
Recall the definition of the cost sensitive loss:
\begin{equation} 
 L_{\D}^{\cost}(h) = \E_{x \sim \D}\Bigl[\cost(h(x),f(x))\Bigr]\,.
\end{equation} 
Given the the distribution $\D$ on $\X$, the Bayes optimal prediction rule is given by:
\begin{align}
    \arg\min_{i\in\Y} \sum_{j\in \Y} w(i,j)\Pr[f(x)=j\mid x]
\end{align}
 for every $x\in\X$. The early works include \cite{Elkan01}, which characterizes the Bayes optimal predictor as a threshold, and its implication for re-balancing and decision tree learning. In fact, this resembles our binary prediction rule.

Sample complexity for cost-sensitive leaning for large margin appears in \cite{KarakoulasS98}. Additional sample complexity bounds, for cost-sensitive learning, based on transformation of the learning algorithms and using rejection sampling is found in \cite{zadrozny2003cost}.
The idea of cost sensitive learning has a long history in statistics. For binary classification, there are many different metrics used for evaluation. Let 
$w=\begin{psmallmatrix} w_{++} & w_{+-}\\w_{-+} & w_{--} \end{psmallmatrix}$. The false-positive (FP) is $w_{+-}$, the false-negative is $w_{-+}$, the precision is $w_{++}/(w_{++}+w_{+-})$, the recall is $w_{++}/(w_{++}+w_{-+})$, and more.

\subsection{Boosting cost-sensitive loss}

There has been significant amount of work with the motivation of adapting AdaBoost to a cost-sensitive loss. 
At a high level, the different proposals either modify the way the algorithm updates its weights, taking in to account the cost-sensitive loss, or changes the final prediction rule. \citet{NikolaouEKFB16} give an overview on various AdaBoost variants for the cost-sensitive case. See also \cite{Landesa-Vazquez15b}. Our modified update rule of Hedge in \Cref{alg:binary_boost} corresponds to CGAda of \citet{SunKWW07} and related AdaBoost variants.

The theoretically sound proposal focused on deriving similar guarantees to that of AdaBoost.
Cost-sensitive boosting by \cite{Masnadi-ShiraziV11} modified the exponential updates to include the cost-sensitive loss.
Cost-Generalized AdaBoost by \cite{Landesa-VazquezA12} modifies the initial distribution over the examples.
AdaboostDB by \cite{Landesa-VazquezA13} modifies the base of the exponents used in the updates. All the theoretically sound proposal are aimed to guarantee convergence under the standard weak-leaning assumption. Their main objective is to derive better empirical results when faced with a cost-sensitive loss.
However, they do not address the essential question, \emph{when is boosting possible?} In particular, they do not study cost-sensitive variants weak learners and do not characterize the boostability of such learners.
In addition to the papers above, there have been many heuristic modifications of AdaBoost which try to address the cost-sensitive loss \citep{KarakoulasS98,FanSZC99,Ting00,ViolaJ01,SunKWW07}.

\subsection{Multi-objective learning and boosting}
Learning with multiple objectives is also common in machine learning. A well studied special case are variants of learning with \emph{one-sided error} \citep{kivinen1995learning, sabato2012multi}. A typical goal in one-sided learning or the related \emph{positive-reliable} learning \citep{kalai2012reliable,kanade2014distribution,durgin2019hardness} is to guarantee an (almost) $0$ false positive loss and simultaneously a low false negative loss. This corresponds to a $(\costVec, \bbz)$-learner with $\costVec=(\cost_+,\cost_-)$ and $\bbz=(0,\eps)$. More generally, \citet{kalai2012reliable} also considered \emph{tolerant reliable} learning with an arbitrary $\bbz=(z_+,z_-)$. Our results apply in this context. For example in the binary case, our results show that a $(\costVec, \bbz)$-learner (i.e., a $\bbz$-tolerant-reliable one) is boostable if and only if $\sqrt{z_+}+\sqrt{z_-} < 1$. Moreover, our results also imply boostability and learnability results for reliable learning in the multi-class case; a learning setting mostly over-looked so far.

\section{Multiclass Boosting}
Boosting is a fundamental methodology in machine learning 
    that can boost the accuracy of weak learning rules into a strong one.
    Boosting theory was originally designed for binary classification. The study of boosting was initiating in a line of seminal works which include the celebrated Adaboost algorithm, as well an many other algorithms with various applications, (see e.g.~\cite{Kearns88unpublished,Schapire90boosting,Freund90majority,Freund97decision}). It was later adapted to other settings and was extensively studied in broader contexts as well~\citep{bendavid2001agnostic,KalaiS05,LongS08,kalai2008agnostic,kanade2009potential,feldman2009distribution,moller2024many,green2022optimal,raman2022online,raman2019online}.
 
    There are also various extension of boosting to the multiclass setting.
    The early extensions include AdaBoost.MH, AdaBoost.MR, and approaches based on Error-Correcting Output Codes (ECOC) \citep{schapire1999_MR, allwein2000reducing}. These works often reduce the $k$-class task into a single binary task. The binary reduction can have various problems, including increased complexity, and lack of guarantees of an optimal joint predictor.  

Notably, a work by \cite{Mukherjee2013} established a theoretical framework for multiclass boosting, which generalizes previous learning conditions. However, this requires the assumption that the weak learner minimizes a 
complicated loss function, that is significantly
different from simple classification error.

More recently, there have been several works on multiclass boosting.  First, \cite{brukhim2021multiclass} followed a formulation for multiclass boosting similar to that of \cite{Mukherjee2013}. They proved a certain hardness result showing that a broad, yet restricted, class of boosting algorithms must incur a cost which scales polynomially with $|\Y|$. Then, \cite{brukhim2023improper} and \cite{Brukhim23simple} demonstrated efficient methods for multiclass boosting. We note that our work generalizes the results given in \cite{Brukhim23simple} to the cost sensitive setting, as detailed in \Cref{sec:multiclass}.